\documentclass[accepted]{uai2021}

\usepackage{amsmath,amsfonts,bm}

\def\floor#1{\lfloor #1 \rfloor}
\def\1{\bm{1}}

\DeclareMathAlphabet{\mathsfit}{\encodingdefault}{\sfdefault}{m}{sl}
\SetMathAlphabet{\mathsfit}{bold}{\encodingdefault}{\sfdefault}{bx}{n}

\def\gB{{\mathcal{B}}}
\def\gC{{\mathcal{C}}}

\def\gF{{\mathcal{F}}}
\def\gG{{\mathcal{G}}}
\def\gH{{\mathcal{H}}}

\def\gN{{\mathcal{N}}}

\def\gY{{\mathcal{Y}}}

\newcommand{\R}{\mathbb{R}}

\usepackage{amsmath,amssymb,amsthm}
\usepackage{bm}
\usepackage{graphicx,here}

\newtheorem{theorem}{Theorem}
\newtheorem{corollary}{Corollary}
\newtheorem{lemma}{Lemma}
\newtheorem{proposition}{Proposition}

\theoremstyle{definition}
\newtheorem{definition}{Definition}
\newtheorem{remark}{Remark}
\newtheorem{example}{Example}

\renewcommand{\hat}{\widehat}
\renewcommand{\tilde}{\widetilde}

\usepackage{color}

\usepackage{microtype}
\usepackage{graphicx}
\usepackage{subfigure}
\usepackage{booktabs} 

\allowdisplaybreaks

\newcommand{\vertiii}[1]{{\left\vert\kern-0.25ex\left\vert\kern-0.25ex\left\vert #1 
    \right\vert\kern-0.25ex\right\vert\kern-0.25ex\right\vert}}

\usepackage{hyperref}

\usepackage[utf8]{inputenc} 
\usepackage[T1]{fontenc}    
\usepackage{hyperref}       
\usepackage{url}            
\usepackage{booktabs}       
\usepackage{amsfonts}       
\usepackage{nicefrac}       
\usepackage{microtype}      

\usepackage[american]{babel}

\usepackage{natbib} 
\usepackage{mathtools} 
\usepackage{booktabs} 
\usepackage{tikz} 

\title{Improved Generalization Bounds of Group Invariant / Equivariant \\Deep Networks via Quotient Feature Spaces}

\author[1,2]{Akiyoshi Sannai} 
\author[3,1]{Masaaki Imaizumi}
\author[3]{Makoto Kawano}
\affil[1]{%
    RIKEN Center for Advanced Intelligence Project\\
    Chuo, Tokyo, Japan
}
\affil[2]{%
    Keio University\\
    Minato, Tokyo, Japan
}
\affil[3]{
    The University of Tokyo\\
    Bunkyo, Tokyo, Japan
}

\begin{document}
\maketitle

\begin{abstract}
 Numerous invariant (or equivariant) neural networks have succeeded in handling invariant data such as point clouds and graphs. However, a generalization theory for the neural networks has not been well developed, because several essential factors for the theory, such as network size and margin distribution, are not deeply connected to the invariance and equivariance. In this study, we develop a novel generalization error bound for invariant and equivariant deep neural networks. To describe the effect of invariance and equivariance on generalization, we develop a notion of a \textit{quotient feature space}, which measures the effect of group actions for the properties. Our main result proves that the volume of quotient feature spaces can describe the generalization error. Furthermore, the bound shows that the invariance and equivariance significantly improve the leading term of the bound. We apply our result to specific invariant and equivariant networks, such as DeepSets \citep{zaheer2017deep}, and show that their generalization bound is considerably improved by $\sqrt{n!}$, where $n!$ is the number of permutations. We also discuss the expressive power of invariant DNNs and show that they can achieve an optimal approximation rate. Our experimental result supports our theoretical claims. 
\end{abstract}

\section{Introduction}\label{sec:intro}
Group invariant (or equivariant) deep neural networks have been extensively utilized in data analysis \citep{shawe1989building,shawe1993symmetries,ntampaka2016dynamical,ravanbakhsh2016estimating,faber2016machine,cohen2016group,zaheer2017deep,li2018so,su2018splatnet,li2018pointcnn,yang2018foldingnet,xu2018spidercnn,lenssen2018group,cohen2019general}.
A typical example is permutation invariant deep neural networks for point cloud data.
The data are given as a set of points, and permuting points in the data does not change the result of its prediction \citep{zaheer2017deep,li2018so,su2018splatnet,li2018pointcnn,yang2018foldingnet,xu2018spidercnn, ntampaka2016dynamical,ravanbakhsh2016estimating,faber2016machine}.
Another example is graph neural networks for graph data, that are represented by a column-and-row-permutation invariant adjacency matrix \citep{bruna2013spectral,henaff2015deep,monti2017geometric,ying2018hierarchical}.
The group invariant and equivariant neural networks can significantly improve the accuracy of prediction with limited data size and network size \citep{zaheer2017deep,li2018pointcnn,li2018so}.
Their theoretical properties have been investigated as well.
Universal approximation properties are proved for several invariant and equivariant neural networks \citep{yarotsky2018universal,maron2019universality,sannai2019universal,segol2019universal,ravanbakhsh2020universal}.

Despite the impact and high empirical accuracy, the generalization error of group invariant / equivariant neural networks has not been well clarified yet.
This is because there are several theoretical difficulties in connecting invariance with generalization theory.
First, invariance is not strongly connected to common factors that are important to the theory.
The generalization error bounds of ordinary deep neural networks are mainly controlled by their depth, width, number of trainable parameters, and margin distributions \citep{anthony2009neural,neyshabur2015norm,bartlett2017spectrally}.
However, invariance and equivariance are determined independently of these factors. 
Second, there are few quantitative features which can assess invariance and equivariance.
Without a quantitative criterion, it is not possible to measure how invariance and equivariance affect on generalization errors.


In this study, we establish a unified generalized error bound by developing a quantitative measure to describe the effects of invariance and equivariance.
For a deep neural network $f$, let $R(f)$ be its expected loss and $R_m(f)$ be its empirical loss with $m$ training samples.
For a set of neural networks $\gF$, we are interested in the following value
\begin{align}
    \gG(\gF) := \sup_{f \in \gF} |R(f) - R_m(f)|, \label{def:ggap}
\end{align}
which is referred as a bound on \textit{generalization gap} or \textit{generalization error}.
Our theory can describe significant improvements in the generalization error bounds of invariant / equivariant neural networks.
We summarize our results as follows.

\textit{(i) Generalization Bound with Quotient Feature Space}:
We develop a notion of a  \textit{quotient feature space} (QFS) and prove that the generalization error bound of invariant / equivariant neural networks is described by the volume of QFSs.
For a finite group $G$, we define a {quotient feature map} $\phi_G : \mathbb{R}^n \to \mathbb{R}^n/G$ and then define a QFS as
    $\Delta_G := \phi_G([0,1]^n)$,
which is regarded as a feature space associated with $G$.
Our results show that the generalization error is proportional to the square root of the volume of $\Delta_G$ (invariant case) or $\Delta_{\mathrm{St}(G)}$ (equivariant case), where $\mathrm{St}(G) \subset G$ is a subgroup of elements whose first coordinates are fixed, named a \textit{stabilizer subgroup}.
In short, with a set of $G$-invariant deep neural networks $\gF^G$, we obtain the following intuition:
\begin{align*}
    \gG(\gF^G) \propto \sqrt{\mathrm{vol}(\Delta_G)}
\end{align*}
Theorem \ref{thm:main1} shows a rigorous statement.
Figure \ref{fig:domains} provides examples of QFSs with several $G$.

\textit{(ii) Roles of Invariance for Generalization}:
We identify how invariance improves generalization through the result with QFSs.
We consider the symmetric group $ G=S_n$ for example.
In this case, we derive the following bound:
\begin{theorem}[Informal Corollary \ref{cor:sn}]
Let $\gF^{S_n}$ be a set of $S_n$-invariant deep neural networks.
For any $\varepsilon> 0$,
    \begin{align*}
        \gG(\gF^{S_n}) \leq  O \left(\sqrt{\frac{1 }{n!~m^{2/n}}} \right)+ {\sqrt{\frac{2\log  ( 1/2\varepsilon )}{m}}},
    \end{align*}
    holds with at least probability $1-2\varepsilon$.
\end{theorem}
This bound reveals two properties of invariant networks.
First, the scale of the bound is improved by $\sqrt{n!}$. 
This result follows the fact that the generalization gap is proportional to the size of QFSs.
This improvement is significant, since $n$ takes a large value in recent point cloud data, for example $n > 1,000$.
Second, it slows down the convergence rate in a number of samples $m$.
This deterioration is a price of gaining the factorial improvement in $n$.
However, as shown in Figure \ref{fig:bound_oridinal}, the factorial improvement greatly outweighs the rate deterioration.

\begin{figure}
    \centering
    \includegraphics[width=0.8\hsize]{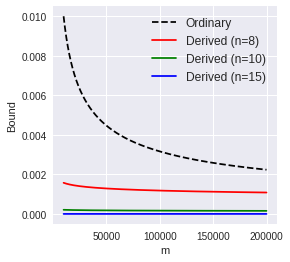}
\vskip -0.1in
    \caption{Order of the bound for the generalization gap against $m$. \textit{Ordinary} (dashed line) denotes $(1/\sqrt{m})$ without invariance, and \textit{Derived} (colored lines) denote the bound $( 1/\sqrt{n! m^{2/n}})$ with $n \in \{8,10,15\}$. Regardless of the effect of $m$, the derived bound gets tight sharply as $n$ increases. }
    \label{fig:bound_oridinal}
\end{figure}



\begin{figure*}[t]
    \centering
      \begin{minipage}{0.24\hsize}
        \centering
          \includegraphics[width=0.99\hsize]{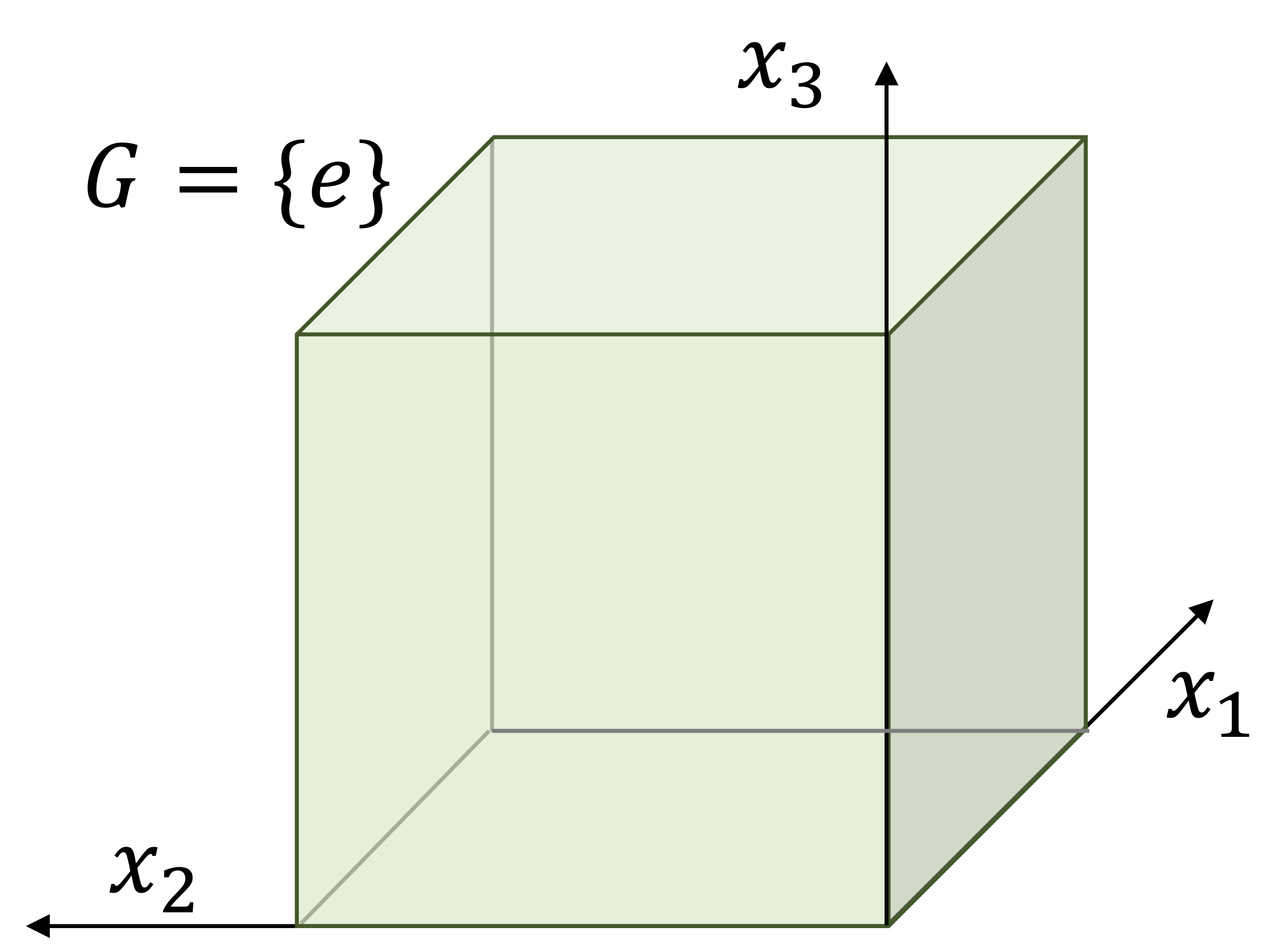}
      \end{minipage}
            \begin{minipage}{0.24\hsize}
        \centering
          \includegraphics[width=0.99\hsize]{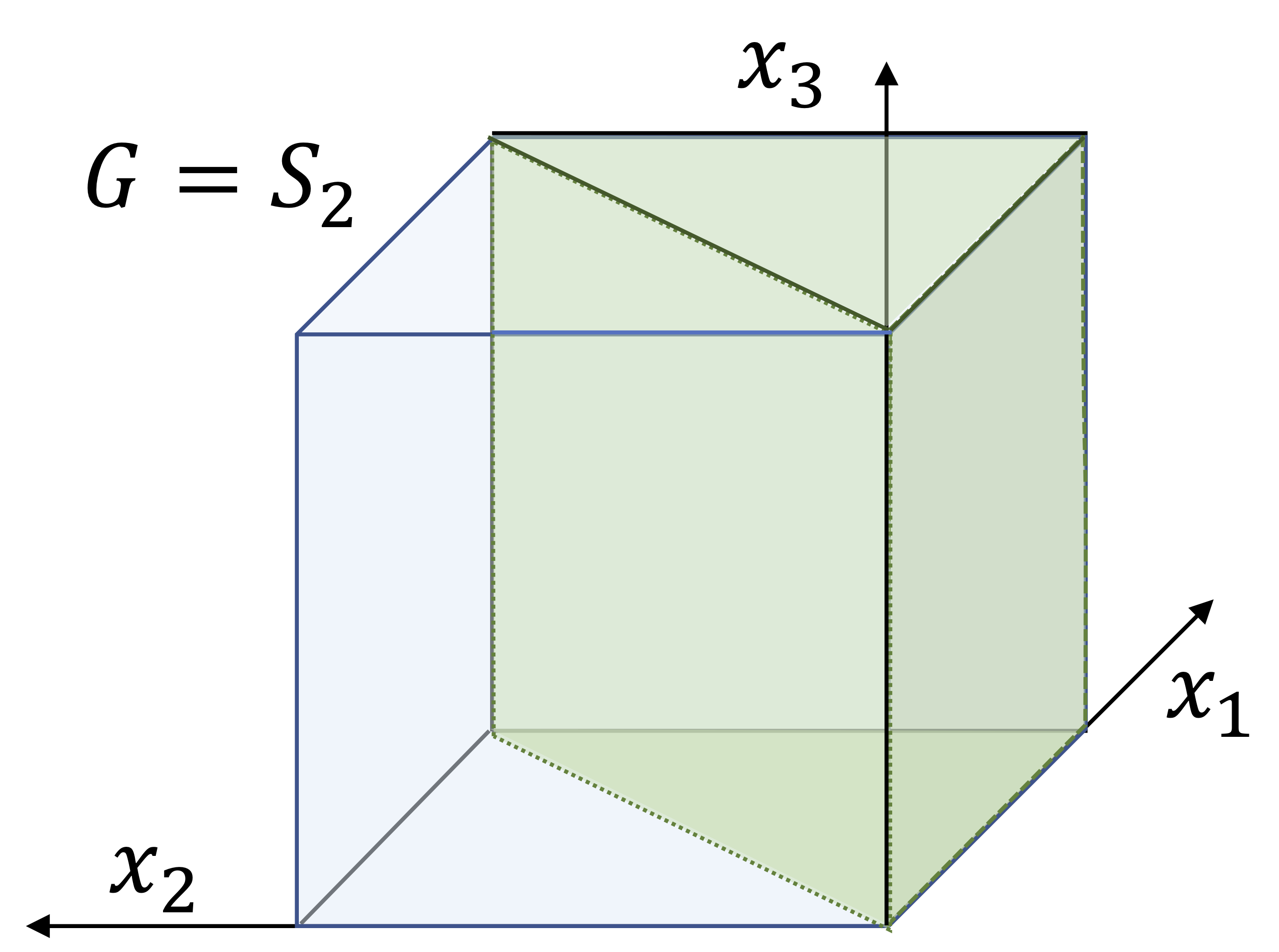}
      \end{minipage}
            \begin{minipage}{0.24\hsize}
        \centering
          \includegraphics[width=0.99\hsize]{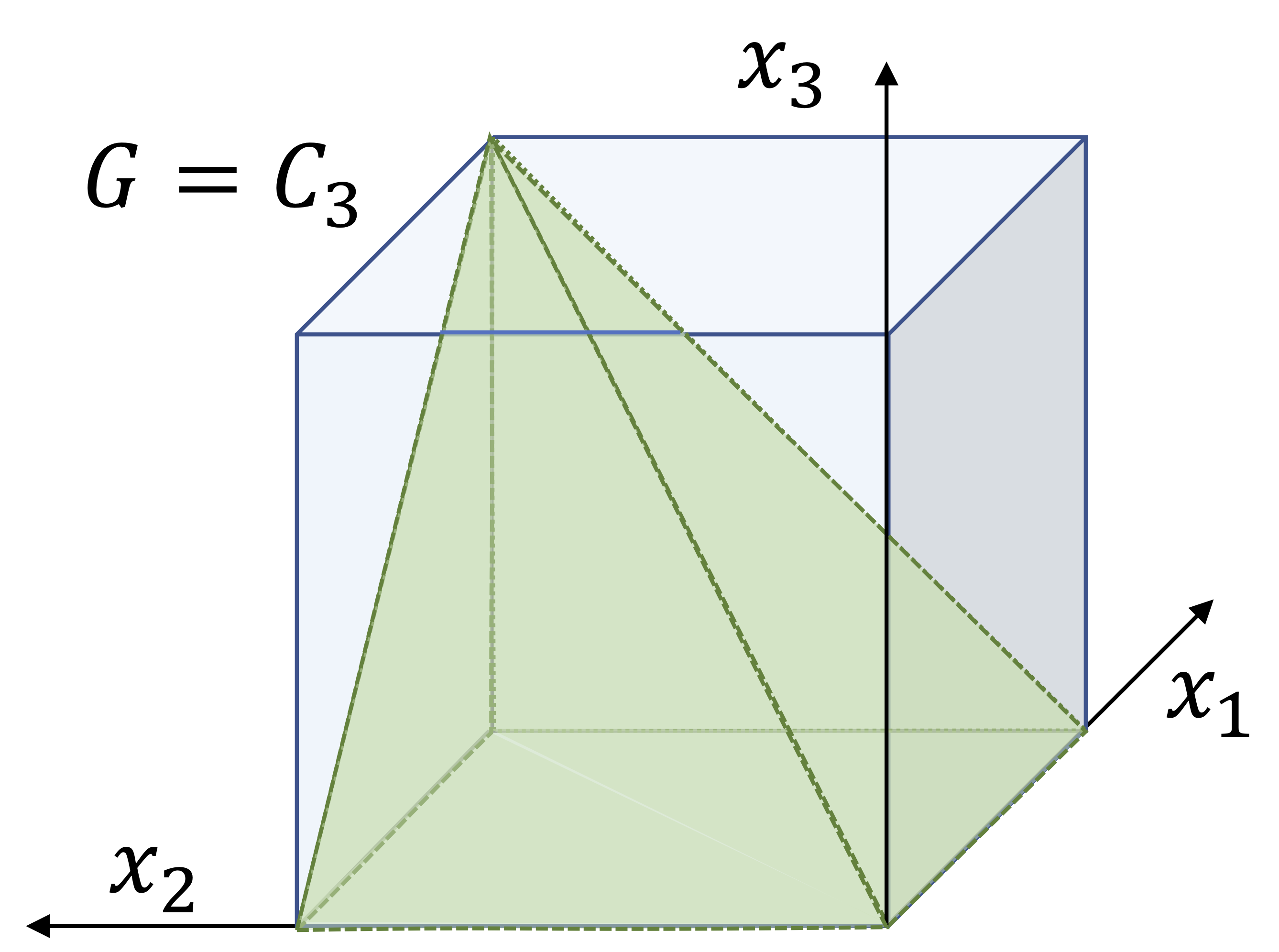}
      \end{minipage}
            \begin{minipage}{0.24\hsize}
        \centering
          \includegraphics[width=0.99\hsize]{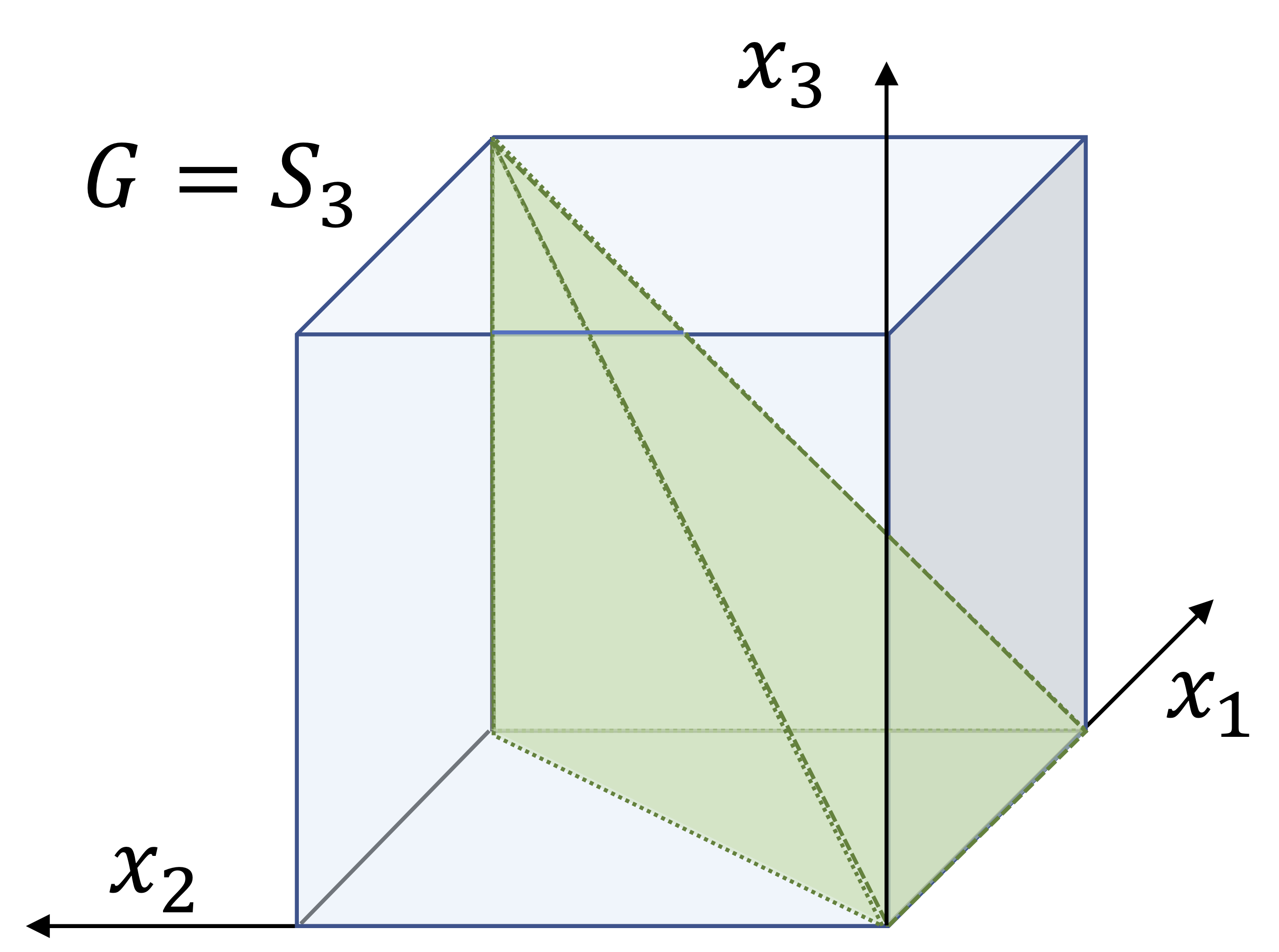}
      \end{minipage}
      \caption{Example of quotient feature spaces with $n=3$. (i) trivial group case ($G=\{e\}$): $\Delta_{\{e\}} = [0,1]^3$. (ii) symmetric group case ($G=S_2$): $\Delta_{S_2} = \{x \in [0,1]^3 \mid x_1 \geq x_2\}$. (iii) cyclic group case ($G=C_3$):  $\Delta_{C_3} = \{x \in [0,1]^3 \mid x_1 \geq x_2 \geq x_3\}\cup \{x \in [0,1]^3 \mid x_1 \leq x_2 \leq x_3\}$. (iv) symmetric group case ($G=S_3$): $\Delta_{S_3} = \{x \in [0,1]^3 \mid x_1 \geq x_2 \geq x_3\}$. \label{fig:domains}}
\end{figure*}

We have mainly two technical contributions.
First, we define the notion of a QFS and show its geometric properties, then derive its volume with a wide class of $G$.
Second, we show a connection between a set of invariant / equivariant neural networks and the volume of QFSs, then describe their generalization bounds by the volume.
Furthermore, we investigate the expressive power of $S_n$-invariant deep neural networks and show their expressive power attains an optimal rate.

\begin{table*}[t]
    \centering
    \begin{tabular}{cccc} \hline 
        Deep Network & Group  & $\mathrm{vol}(\Delta_{G})$ & $\mathrm{vol}(\Delta_{\mathrm{St}(G)})$ \\ \hline
        Deep Sets \citep{zaheer2017deep} & $S_n$  & $O(1/(n!))$ & $O(1/((n-1)!))$ \\
         $G$-CNN \citep{cohen2016group} & $C_4$ &  $O(1/4)$  &  $O(1)$ \\
        Graph Network \citep{maron2018invariant}& $S_n \subset S_{n^2}$   & $O(1/(\# \mbox{of nodes})!)$ & $O(1/((\# \mbox{of nodes}) -1)!)$ \\
        Tensor Network \citep{maron2019universality} & $G \subset S_n$  & $O(1/|G|)$ & $O(1/|\mathrm{St}(G)|)$ \\
 DSS \citep{maron2020learning} & $S_n\times G'~(G' \subset S_N)$  & $O(1/(n!|G'|))$ & $O(1/((n-1)!|\mathrm{St}(G')|)$ \\\hline
    \end{tabular}
    \caption{ Examples of invariant / equivariant DNNs utilized in practice. $S_n$ denotes a symmetric group of order $n!$, and $C_n$ denotes a cyclic group of order $n$. $G$ denotes a subgroup of the permutation group $S_n$ of axes of the input space $\R^n$. $G'$  denotes a subgroup of the permutation group $S_N$ of axes of the input space $\R^{n \times N}$.
 We set $\mathrm{vol}(\Delta_G) = \gN_{\varepsilon,\infty}(\Delta_G)$, where $\gN_{\varepsilon,\infty}(\Delta_G)$ is a covering number of $\Delta_G$ in terms of $\|\cdot\|_\infty$. DSS was referred to as ``Deep Sets for Symmetric elements layers'' \citep{maron2020learning}. 
    \label{tab:list_dnn}}
\end{table*}

\subsection{Related Work}
There are several works studying the generalized performance and sample complexity of neural networks with invariance / equivariance.
\cite{shawetaylor1995sample} shows that the sample complexity increases by a number of equivalent classes.
The closest work with our study is  \cite{sokolic2017generalization}, which considers a general algorithm for the classification problem.
Their generalization bound is proportional to $\sqrt{1/T}$, where $T$ is the number of \textit{transformations} generated by the invariance property.
While their research is excellent, we improve their work in two ways.
(I) Our result has a more concrete structure: our generalization bound describes an explicit role of invariance and equivariance through the notion of QFSs.
Owing to QFSs, our result can be applied to various cases such as graphs.
(II) We relax their strong assumptions on stability and provide accurate analysis.
We provide its detail in Section \ref{sec:relation}.
In fact, our theoretical results are not limited to deep neural networks.
However, most of the models that can control invariant data with large $n$, such as point clouds and large graphs, are mainly handled by deep neural networks \citep{zaheer2017deep,maron2018invariant,maron2019universality,maron2020learning}.
Hence, we regard neural networks as the main application of our theory.
\subsection{Notation}
For a vector $b \in \R^D$, its $d$-th element is denoted by $b_{d}$.
For a function $f:\Omega \to \R$ with a set $\Omega$, $\|f\|_{L^q}:= ( \int_\Omega |f(x)|^q dx)^{1/q}$ denotes the $L^q$-norm for $q \in [0,\infty]$. 
For a subset $\Lambda \subset \Omega$, $f_{\restriction_{\Lambda}}$ denotes a restriction of $f$ to $\Lambda$.
$C(\Omega)$ denotes a set of continuous functions on $\Omega$.
For an integer $z$, $z!= \prod_{j=1}^n j$ denotes a factorial of $z$.
For a set $\Omega$, $\mbox{id}_{\Omega}$ or $\mbox{id}$ denotes the identity map on $\Omega$, namely $\mbox{id}_{\Omega}(x) = x$ for any $x \in \Omega$.  
For a subset $\Delta \subset \R^n$, $\mbox{int}(\Delta)$ denotes a set of inner points of a set $\Delta$. 
For metric spaces $\Delta$ and $\Delta'$, $\Delta \cong \Delta'$ denotes they are isomorphic as metric spaces. 
The supplementary material maintains all full proofs.

\section{Definition and Problem Setting} \label{sec:setting}

\subsection{Invariance / Equivariance and Deep Neural Network}

We provide a general concept of the invariance and equivariance of functions.
Throughout this paper, we consider a finite group $G \leq S_n$, where $S_n$ denotes the symmetric group.
\begin{definition}[Invariant / Equivariant Function]
For a group $G$ acting on $\R^n$ and $\R^M$, a function $f \colon \R^n \to \R^M$ is 
\begin{itemize}
  \setlength{\parskip}{0cm}
  \setlength{\itemsep}{0cm}
    \item {\em $G$-invariant} if $f(g\cdot x)=f(x)$ holds for any $ g \in G$ and any $x \in \R^n$,
    \item {\em $G$-equivariant} if  $f(g\cdot x)=g \cdot f(x)$ holds for any $ g \in G$ and any $x \in \R^n$.  
\end{itemize}
\end{definition}
For a set $\Omega$, $C^G(\Omega)$ denotes a set of $G$-invariant an continuous functions on $\Omega$.



We formulate deep neural networks (DNNs) with invariance and equivariance.
In this study, we consider fully connected DNNs with the ReLU activation function $\mbox{ReLU}(x)=\max(0,x)$.
Let us consider a layer-wise map $Z_i:\mathbb{R}^{d_i}\to \mathbb{R}^{d_{i+1}}$ defined by $Z_i(x) =  \mbox{ReLU}(W_ix+b_i)$, where $W_i \in \mathbb{R}^{d_{i+1}\times d_i}$ and $b_i \in \mathbb{R}^{d_{i+1}}$ for $i=1,...,H$.
Here, $H$ is a depth, and $d_i$ is a width of the $i$-th layer.
Then, a function by DNNs has the following formulation
\begin{align}
    f(x) := Z_H \circ Z_{H-1} \ldots Z_2 \circ Z_1(x). \label{def:dnn}
\end{align}
Further, let $\gF$ be a set of functions with the form \eqref{def:dnn}.

We define a function by invariant and equivariant DNNs.
\begin{definition}[Invariant / Equivariant Deep Neural Network]
    A function $f \in \gF$ is a \textit{$G$-invariant / equivariant DNN}, if $f$ is a $G$-invariant / equivariant function.
\end{definition}
This definition is a general notion and represents several explicit invariant DNNs.
We provide several representative examples as follows.
\begin{example}[Deep Sets]
    A permutation-invariant ($S_n$-invariant) DNN was developed by \cite{zaheer2017deep}.
    Its architecture has $J$ middle permutation-equivariant ($S_n$-equivariant) layers $Z_1,...,Z_J$, a permutation-invariant linear layer $Z_L$, and a fully-connected layer $Z_{F}$.
    Each equivariant layer maintains a parameter matrix
 $W_i = \lambda \bm{I} + \gamma (\bm{1}\bm{1}^{\top}), ~\lambda, \gamma \in \mathbb{R}, \bm{1}= [1,...,1]^{\top}$,
 which makes $Z_j$ be equivariant.
 Then, a DNN $f = Z_F \circ Z_L \circ Z_J \circ \cdots Z_1$ is a permutation-invariant DNN.
\end{example}

\begin{example}[Tensor Network]
For a finite group $G \subset S_n$, 
a $G$-invariant / equivariant DNN was developed by \cite{maron2019universality} using a notion of higher-order tensors.
The study considered a tensor $W \in \mathbb{R}^{n^k \times a}$ and an action $g \in G$ on the tensor as
$(g \cdot W)_{i_1,..,i_k,j} = W_{g^{-1}(i_1),..,g^{-1}(i_k),j}$,
for $i_{k'}=1,...,n,k'=1,...,k$, and $j=1,...,a$.
With the action, the study developed a $G$-invariant / equivariant DNN.
Since $G$ is a finite group, the model is a specific case of our setting.
\end{example}

\subsection{Formulation of Learning Problem}

We formulate our learning problem with DNNs.
Let $I = [0,1]^{n}$ be an input space with $n \in \mathbb{N}$ and $\R^M$ be an output space with $M \in \mathbb{N}$.
Let $L:\R^M \times \R^M \to \R$ be a loss function which satisfies $\sup_{y,y' \in \gY}|L(y,y')| \leq 1$ and $1$-Lipschitz continuous.
Let $P^*(x,y)$ be a distribution on $I \times \R^M$ which generate data, and $R(f) = \mathbb{E}_{(X,Y)\sim P^*}[L(f(X),Y)]$ for $f : I \to \R^M$ be the expected loss of $f$.
Suppose we have a training dataset $ \{(X_1,Y_1),...,(X_m,Y_m)\}$ of size $m$ which is independently generated from $P^*$.
Let $R_m(f) := m^{-1}\sum_{i=1}^m L(f(X_i),Y_i)$ be an empirical loss with $f$.
Our interest is to bound $\gG(\gF)$ as \eqref{def:ggap} which illustrates how minimizing $R_m(f)$ on $\gF$ affects $R(f)$.

\section{Quotient Feature Spaces} \label{sec:qfs}

We provide a notion of a \textit{quotient feature space} (QFS), which is a key factor in connecting invariance / equivariance and generalization.
With a quotient space $\R^n/G$ with $G$, we consider a map 
\begin{align*}
    \phi_G: \mathbb{R}^n \to \mathbb{R}^n/G \mbox{~such~as~} \phi_G(x) = \{g \cdot x \mid ~\forall g \in G  \},
\end{align*}
named a \textit{quotient map}.
By the definition of $g$, such $\phi_G$ always exists.
With this notion, we define a QFS.
\begin{definition}[Quotient Feature Space]
    For a finite group $G$, a \textit{quotient feature space} is defined as
    \begin{align*}
        \Delta_G := \phi_G(I).
    \end{align*}
\end{definition}
We can regard a QFS as a feature space with $G$.
We prove that a QFS can equip a distance if $g \in G$ preserves a distance in $I$, which is a fundamental property of feature spaces.
\begin{proposition}[Distance on QFS] \label{prop:dist}
For a finite group $G$, we define a function $d_G: \R^n/G \times \R^n/G \to \R_{\geq 0}$ as
\begin{align*}
    d_G(y, y') = \inf \{ \|x- x'\|_2 \mid \phi_G(x)=y, \phi_G(x')=y'\}.
\end{align*}
Then, $d_G$ is a distance on $\mathbb{R}^n/G$. 
\end{proposition}

Intuitively, the distance $d_G$ for $\R^n/G$ is an infimum of a sum of pairwise distances of points $\{x \mid \phi_G(x)=y\}$ and $\{x' \mid \phi_G(x')=y'\}$.
$g \in G$ maintains the distance when $G$ is a finite group.
We also remark that this proposition does not hold for some infinite groups.

\subsection{Volume Measurement of QFS}

We measure volume of $\Delta_G$, which is a critical factor for a generalization bound of invariant / equivariant DNNs.
We consider two cases: (i) the symmetric group $G=S_n$, and (ii) a finite group $G$.
We measure the volume of a set $\Omega$ using a \textit{covering number} $\gN_{\varepsilon,\infty}(\Omega) :=\inf\{N \mid \exists \{x_j\}_{j=1}^N \subset \Omega, \mbox{~s.t.~} \cup_{j=1}^N \{x \mid \|x-x_j\|_\infty \leq \varepsilon\} \supset \Omega\}$.

\subsubsection{Symmetric Group Case}

We begin with the symmetric group $G=S_n$.
It is convenient to study $S_n$ as a first step, because we can derive an explicit formulation of $\phi_{S_n}$ and $\Delta_{S_n}$.
With the case, an action $\sigma \in S_n$ is a permutation of indexes of $x = (x_1,...,x_n) \in I$.
For $i=1,...,n$, we define a map $\max_i(\{x_1,...,x_n\})$ which returns the $i$-th largest element of $\{x_1,...,x_n\}$.
\begin{proposition}[QFS of $S_n$]\label{prop:delta_sn}
    Define a set $\Delta \subset I$ as $\Delta := \left\{ x \in I \mid x_{1} \geq x_{2} \geq \cdots \geq x_{n}  \right\},$
    and a map $\phi: \R^n \to \R^n$ as
$ \phi(x) :=  (\max_1(\{x_1,..,x_n\}),...,\max_n(\{x_1,..,x_n\}))$.
    Then, we obtain $\phi(\Delta) \cong \Delta_{S_n}$.
\end{proposition}
Figure \ref{fig:domains} illustrates $\Delta_{S_n}$ for some $n$.
Intuitively, any element of $I$ corresponds to some element of $\Delta_{S_n}$ with an existing action $\sigma \in {S_n}$, namely, $I = \cup_{\sigma \in S_n} \left\{ \sigma \cdot x \mid x \in \Delta \right\}$ holds.
With the help of the explicit formulation of $\Delta_{S_n}$, we can measure its size.
Since $\Delta_{S_n} \subset I$ holds, we can measure its volume by the Euclidean distance as follows:
\begin{lemma}[Volume of $\Delta_{S_n}$] \label{lem:delta}
    There is a constant $C$ such that for small enough $\varepsilon >0$, we obtain
    \begin{align*}
        \gN_{\varepsilon,\infty}( \Delta_{S_n}) \leq C /  (n!~\varepsilon^{n} ).
    \end{align*}
\end{lemma}
Lemma \ref{lem:delta} provides an important claim: the volume of $\Delta_{S_n}$ is proportional to $1/(n!)$, i.e., the volume significantly decreases with the increases in $n$. 
The term $\varepsilon^{-n}$ is usual for covering numbers, i.e., $\gN_{\varepsilon,\infty}( I) \leq C/\varepsilon^n$ holds, hence the factorial improvement by $1/(n!)$ comes from $S_n$-invariance.

\subsubsection{General Finite Group Case}

We consider a general finite group $G$ and its corresponding QFS, by studying $\Delta_G$ and measuring its covering volume.
We first prepare several notions.
For a group $G$, $|G|$ denotes its number of elements, named an \textit{order} of $G$.
For a subgroup $H  \subset G$, a set $\{g_1,..,g_K | g_k \in G\}$ is defined as a \textit{complete system of representatives of} $H\backslash G$  if $K=|G|/|H|$ and  $G= \cup_{k=1}^K H\cdot g_k$ hold.
For any $G$ and $H$, we can always find the complete system.
Also, we define $\Delta_k:=g_k \cdot \Delta_{S_n} $.
Then, we achieve the following result:
\begin{proposition}\label{prop:surj}
    Let $\{g_1,..,g_K | g_k \in S_n,  k=1,...,K\}$  be a complete system of representatives of  $G\backslash S_n$.
    Then, $\Delta_k\cong \Delta_{S_n}$ holds as metric spaces for all $k=1,...,K$.
    Furthermore,  its induced set
        $\tilde{\Delta}_G := \bigcup_{k=1}^{|S_n|/|G|} \Delta_k$
    satisfies $\phi_G (\tilde{\Delta}_G)= \Delta_G$.
\end{proposition}
Proposition \ref{prop:surj} shows that we can describe $\Delta_G$ by $\tilde{\Delta}_G$ which is a combination of complete systems of representatives of $G\backslash S_n$.
Intuitively, we can define $\tilde{\Delta}_G$ by a union of several transformed $\Delta_{S_n}$. 

We describe an example with $n=3$ and $G=S_2$.
A complete system of representatives of $S_2\backslash S_3$ can be $\{g_1,g_2,g_3 \} \subset S_3$ such that $g_1$ is an identity, $g_2$ is a transposition of the $2$nd and $3$rd elements, and $g_3$ is a cyclic permutation.
In other words, we have $g_3\cdot 1 =2, g_3 \cdot 2 =3$, and $ g_3\cdot 3 =1$. 
Moreover, we have $\tilde{\Delta}_{S_2} = \Delta_{S_2} $.
Then, we can represent $\Delta_{S_2}$ by $\Delta_{g_k}$ with $k=1,2,3$ as $ \Delta_{S_2}= g_1 \cdot \Delta_{S_3} \cup g_2 \cdot \Delta_{S_3} \cup g_3 \cdot \Delta_{S_3}$.
According to Figure \ref{fig:delta_s2}, $ \Delta_{S_2}$ is a union of $\Delta_{S_3}( = g_1 \cdot \Delta_{S_3})$, reflected $\Delta_{S_3}~(=g_2 \cdot \Delta_{S_3})$, and rotated  $\Delta_{S_3}~(=g_3 \cdot \Delta_{S_3})$.

\begin{figure*}[t]
    \centering
    \includegraphics[width=0.98\hsize]{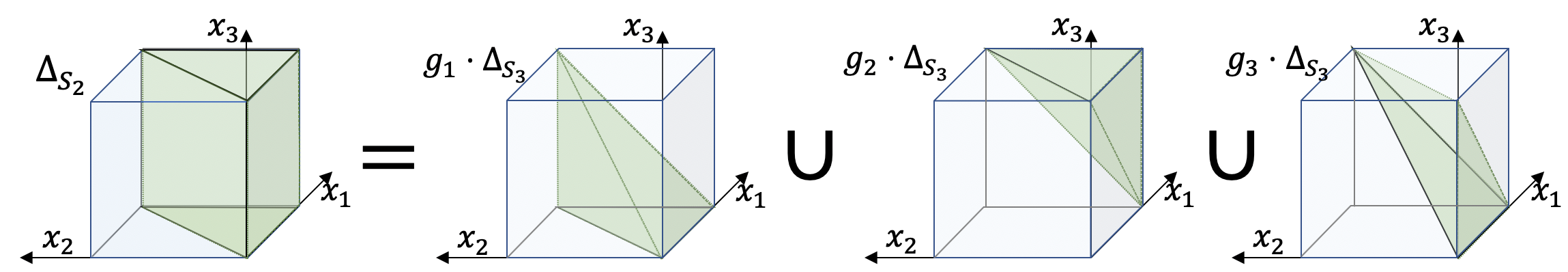}
    \caption{ Illustration  of $ \Delta_{S_2} = \tilde{\Delta}_{S_2} = g_1 \cdot \Delta_{S_3} \cup g_2 \cdot \Delta_{S_3} \cup g_3 \cdot \Delta_{S_3}$. The blue cube is $I$, and the green polyhedrons are ${\Delta}_{S_2}$ and $g_k \cdot \Delta_{S_3}, k=1,2,3$.}
    \label{fig:delta_s2}
\end{figure*}

We can evaluate the volume of $\Delta_G$ by $|G|$ as follows:
\begin{lemma} \label{lem:covering_g}
There exists a constant $C > 0$, such that for small enough $\varepsilon >0$, we obtain
    \begin{align*}
        \gN_{\varepsilon,\infty}( \Delta_G) \leq C/(|G|~\varepsilon^{n} ).
    \end{align*}
\end{lemma}
Similar to Lemma \ref{lem:delta}, the result of Lemma \ref{lem:covering_g} states that the covering volume of $\Delta_G$ is improved by $|G|$.
Since $|S_n| = n!$ holds, Lemma \ref{lem:covering_g} is a generalization of Lemma \ref{lem:delta}.
Table \ref{tab:list_dnn} contains examples of $G$.

\subsection{Covering Numbers of QFS} \label{sec:func_corresp}

We show several technical inequalities to present a relationship between $G$-invariant DNNs and $\Delta_G$.
Namely, we show that a covering number of a set $\gF^{G}(I) = \{f : I \to \R^M \mid f~\mbox{is a $G$-invariant DNN.} \}$ is evaluated by comparison with the volume of $\Delta_G$.
We also define $\gF(\Delta_G) := \{f : \Delta_G \to \R^M \mid f~\mbox{is a DNN.} \}$.
We note that $f \in \gF(\Delta_G)$ is an ordinary DNN rather than a G-invariant DNN.

First, we derive a corresponding map between the two functional sets.

\begin{proposition}\label{prop:correspondence}
$\phi_G$ induces a bijection $\hat{\phi}_G: C(\Delta_G) \to  C^G(I)$.
Further, $f \in C(\Delta_G)$ is  $K$-Lipschitz continuous if and only if $\hat{\phi}_G(f)$ is $K$-Lipschitz continuous.
\end{proposition}

Using the corresponding map, we evaluate the volume of $ \gF^{G}(I)$ by $\gF(\Delta_G)$.
The following result presents the claim.
\begin{proposition} \label{prop:covering_bound_g}
    For any $\varepsilon > 0$, we obtain
    \begin{align*}
        \gN_{\varepsilon,\infty}( \gF^{G}(I)) \leq \gN_{\varepsilon,\infty}( \gF(\Delta_G)).
    \end{align*}
\end{proposition}
This inequality shows that the set of $G$-invariant DNNs on $I$ is bounded by the volume of the set of DNNs on $\Delta_G$ \textit{without} invariance.

Finally, we evaluate the volume of $\gF(\Delta_G)$ in terms of the volume of $\Delta_G$.
We provide an inequality which bounds the covering number of $\gF(\Delta_G)$ by a \textit{polynomial} of the volume of $\Delta_G$, whereas the commonly used inequality only includes the \textit{logarithm} of the volume of $\Delta_G$, such as the result in Section 10.2 in \cite{anthony2009neural}.
\begin{proposition} \label{prop:covering_bound_delta}
    Suppose that any function in $\gF(\Delta_G)$ is $C_\Delta$-Lipschitz continuous and uniformly bounded by $B$ with constants $C_\Delta, B > 0$.
    Then, with an existing constant $c > 0$ and $C$ in Lemma \ref{lem:delta}, for any $\delta > 0$, we obtain
    \begin{align*}
        &\log \gN_{2 C_\Delta \delta,\infty}( \gF(\Delta_G) )\leq  \gN_{\delta,\infty}( \Delta_{S_n})\log  (8c^2B/\delta).
    \end{align*}
\end{proposition}
Combining this result with Proposition \ref{prop:covering_bound_g}, we can utilize $ \gN_{\varepsilon,\infty}( \Delta_{S_n})$ as the quantitative measure to evaluate the volume of the set of $G$-invariant DNNs.

\begin{remark}[Linear bound in $\gN_{\delta,\infty}( \Delta_{S_n})$]
In Proposition \ref{prop:covering_bound_delta}, it is important to note that logarithm of $\gN_{2 C_\Delta \delta,\infty}( \gF(\Delta_G))$ is linearly bounded by $\gN_{\delta,\infty}( \Delta_{S_n})$.
In general, log of $\gN_{2 C_\Delta \delta,\infty}( \gF(\Delta_G), )$ is bounded by a number of parameters of DNNs \citep{anthony2009neural} or parameter norms \citep{bartlett2017spectrally}. 
However, these values have little to do with invariance and therefore cannot give tight bounds.
We instead consider the volume of $\Delta_{S_n}$ as a value related to invariance and achieve the linear bound  in $\gN_{\delta,\infty}( \Delta_{S_n})$.
\end{remark}

\section{Generalization Bound for Invariant DNNs} \label{sec:gen_inv}

We derive a generalization bound with QFSs and show that invariance can effectively improve the generalization performance of DNNs.
Utilizing the results above, we have the following main result:

\begin{theorem}[Generalization of Invariant DNN] \label{thm:main1}
    Suppose any $f \in \gF^{G} = \gF^{G}(I)$ is uniformly bounded by $1$.
    Then, for any $\varepsilon > 0$, there exist a constant $C> 0$ that are independent of $n, m$ and $\varepsilon$, and the following inequality holds with probability at least $1-2\varepsilon$:
      \begin{align*}
        \gG (\gF^G) \leq \underbrace{\sqrt{\frac{C }{|G|~ m^{2/n}}}}_{=:I_1}+ \underbrace{\sqrt{\frac{ 2 \log (1/2\varepsilon)}{m}}}_{=:I_2}.
    \end{align*}
\end{theorem}
The main term $I_1$ of the bound is interpreted to maintain the relation $I_1 \propto \sqrt{{\gN_{\varepsilon,\infty}( \Delta_G)}}$, hence the volume of QFSs describes the effect of invariance on the generalization error.
Obviously, $I_1$ is improved as  $\sqrt{|G|}$ increases.
Although the convergence rate of the main term in $m$ gets slow as $n$ increases, an increase in $\sqrt{|G|}$ reduces the error, as shown in the following specific example.
Here, we note that we can regard $I_2$ as a relatively negligible term.


With the case $G=S_n$, the result in Theorem \ref{thm:main1} yields a more explicit bound:
\begin{corollary}[Generalization of $S_n$-invariant DNN]\label{cor:sn}
    Consider the same setting as Theorem \ref{thm:main1}.
    Then, for any $\varepsilon > 0$, there exists a constant $C > 0$ and the following inequality holds with probability at least $1-2\varepsilon$: 
    \begin{align*}
    &\gG (\gF^{S_n}) \leq {\sqrt{\frac{C }{n!~m^{2/n}}}}+ {\sqrt{\frac{ 2 \log (1/2\varepsilon)}{m}}}.
     \end{align*} 
\end{corollary}
Corollary \ref{cor:sn} follows the order $|S_n| = n!$.
Since $n$ is large in practice, e.g., a number of points in point cloud data or a number of nodes for graph data, the term $n!$ significantly improves the bound.

\begin{remark}[Convergence rate in $m$] \label{remark:rate_m}
    In the result, the convergence rate to the sample size $m$ slows down as $n$ increases, but the improvement of the bound with increasing $n$ more than cancels this out.
    Since a decay by the factorial term $n!$ is sufficiently faster than any polynomial convergence in $n$. 
    For a practical example with an experiment ($m=9843, n=100$) by \cite{zaheer2017deep} with the ModelNet40 dataset \citep{wu20153d}, our bound $O(1/\sqrt{n! m^{2/n}}) \approx O(10^{-156})$ is \textit{significantly tighter} than an ordinary bound $O(1/\sqrt{m}) \approx O(10^{-1.99})$.
\end{remark}


\textit{Proof sketch for Theorem \ref{thm:main1}}:
We prove Theorem \ref{thm:main1} by the following three steps.

First, we apply the well-known Rademacher complexity bound (e.g., Lemma A.5 in \cite{bartlett2017spectrally}) and obtain the following inequality with probability at least $1-2\varepsilon$
\begin{align}
    &\gG(\gF^G) \leq \sqrt{\frac{  2 \log (1/2\varepsilon)}{m}}  \label{ineq:basic_bound} \\
  &+ \inf_{\alpha \geq 0} \left\{ {4\alpha}+ \frac{12}{\sqrt{m}} \int_\alpha^{\sqrt{m}}\sqrt{ 2 \log 2 \gN_{\delta,\infty}(\gF^{G}(I)) }d \delta \right\}. \notag
\end{align}
Second, we bound the term $\log \gN_{\delta,\infty}(\gF^{G}(I))$ in \eqref{ineq:basic_bound} by $\log \gN_{\delta,\infty}(\gF(\Delta_G))$ by using the result in Proposition \ref{prop:correspondence}.
This enables us to evaluate the error with $G$-invariance using $\Delta_G$.

Third, we bound $\log \gN_{\delta,\infty}(\gF(\Delta_G))$ by the term with $\gN_{\delta,\infty}( \Delta_G)$.
To achieve this bound for bounding the volume of functional sets by that of its domain, we provide Proposition \ref{prop:covering_bound_delta} in the supplementary material.
Then, we combine Lemma \ref{lem:covering_g} and get the statement of Theorem \ref{thm:main1}.
\qed

\section{Generalization Bound for Equivariant DNNs} \label{sec:gen_equiv}

We derive a generalization bound for equivariant DNNs.
To this aim, we require a covering number of the following set 
    $\tilde{\gF}^{G}(I) = \{\tilde{f} : I \to \R^n \mid \tilde{f}~\mbox{is a $G$-equivariant DNN.} \}$.

As preparation, we define a \textit{stabilizer subgroup} associated with $G$. 
In this section, for brevity, we consider that the action $G$ is transitive, i.e. for any $i \in \{1,2,...,n\}$, there exists $g \in G$ that satisfies $g\cdot 1 = i$.
We define the stabilizer subgroup $\mathrm{St}(G) \subset G$ as
    $\mathrm{St}(G) = \left\{g\in G \mid g\cdot 1 = 1 \right\}$.
Here, $\mathrm{St}(G)$ is a subgroup of $G$ which fixes the first coordinate.
We utilize this subgroup for decomposing equivariant functions and obtain the following bound:
\begin{theorem}[Generalization of Equivariant DNN] \label{thm:main1_equiv}
    Suppose $G$ is transitive, and any $\tilde{f}^G \in \tilde{\gF}^{G} = \tilde{\gF}^{G}(I)$ is uniformly bounded by $1$.
    Then, for any $\varepsilon > 0$, there exists constant $\tilde{C} > 0$ that are independent of $n,m$ and $\varepsilon$, the following inequality holds with probability at least $1-2\varepsilon$:
  \begin{align*}
        \gG(\tilde{\gF}^{G}) \leq  \underbrace{\sqrt{ \frac{ \tilde{C} }{|\mathrm{St}(G)|~ m^{2/n}}}}_{=: I'_1}+ \underbrace{\sqrt{\frac{2\log  ( 2/\varepsilon )}{m}}}_{=: I'_2}.
    \end{align*}
\end{theorem}
The result shows that equivariant DNNs also achieves the improved generalization bound by the volume of its QFS of $\mathrm{St}(G)$, i.e., the main term satisfies 
    $I_1' \propto \sqrt{{\gN_{\varepsilon,\infty}( \Delta_{\mathrm{St}(G)})}}$.
The remainder term ${I}_2'$ has a smaller order than $I'_1$. Thus, it is considered to be negligible in our analysis.

By Theorem \ref{thm:main1_equiv}, we obtain the following specific generalization bound with $G=S_n$.
\begin{corollary}[Generalization of $S_n$-equivariant DNN]\label{cor:sn_equiv}
    Consider the same setting as Theorem \ref{thm:main1}.
    Then, for any $\varepsilon > 0$ and sufficiently large $n$, the following inequality holds with probability at least $1-2\varepsilon$:
    \begin{align*}
        \gG(\tilde{\gF}^{S_n}) \leq {\sqrt{\frac{\tilde{C}}{(n-1)!~m^{2/n}}}}+ {\sqrt{\frac{ 2 \log (1/2\varepsilon)}{m}}}.
    \end{align*}
\end{corollary}
This corollary describes that $S_n$-equivariant DNNs can be improved bound by $\sqrt{(n-1)!}$.
In Section \ref{sec:equiv_without_transitive} in the supplementary material, we relax the transitive setting for $G$ and provide more general results with non-transitive $G$.

Even with this result, the slow decay rate in $m$ is resolved by the improvement of the bound by $n$ due to invariance.
The detailed discussion is similar to that of Remark \ref{remark:rate_m}.

\textbf{Proof sketch for Theorem \ref{thm:main1_equiv}}:
As a preparation, we define a set of $G$-equivariant functions with multivariate outputs and their covering numbers.
To this end, we reform the $G$-equivariant function $\tilde{f}^G: I \to \R^n$ to a combination of $\mathrm{St}(G)$-invariant functions.
Proposition 3.1 in \cite{sannai2019universal} shows the following formulation:
\begin{align}
    \tilde{f}^G = (f^{\mathrm{St}(G)} \circ \tau_{1,1},  \cdots, f^{\mathrm{St}(G)}\circ \tau_{1,n})^\top, \label{rep:equiv}
\end{align}
where $f^{\mathrm{St}(G)}:I \to \R$ is an existing $\mathrm{St}(G)$-invariant function, and $\tau_{1,j} \in G$ is a linear map for $j = 1,..,n$ such that it makes the first coordinate of an input move to the $j$-th coordinate.
The detailed results are provided in Proposition \ref{prop:relation-invariant-equivariant} in Section \ref{sec:equiv_without_transitive}.
By the representation, we can evaluate a covering number of $\tilde{\gF}^{G}(I)$ by that of $\mathrm{St}(G)$-invariant functions.
For a multi-output function ${f} : I \to \R^n$ as ${f}=(f_1,...,f_n)$, we define a norm $\vertiii{ {f}}_{L^\infty(I)} := \max_{j = 1,...,n} \|f_j\|_{L^\infty(I)}$.
Also, let $\tilde{\gN}_{\varepsilon,\infty} (\Omega)$ be a covering number of $\Omega$ in terms of $\vertiii{\cdot}_{L^\infty(I)}$.

The remaining steps of this proof are similar to those of Theorem \ref{thm:main1}.
\qed

\section{Experimental Result}


We experimentally validate Theorem~\ref{thm:main1} by measuring a generalization gap with synthetic data.
We consider a regression task to find a sum of $n$ scalars, which is a problem solvable by invariant functions.

We generate synthetic data by the following process.
For inputs, we generate $N=nd$ random variables $x_1,...,x_N$  that are independently and identically generated from a standard normal distribution. 
We generate an output variable $y = \sum_{i=1}^{N}x_i$.
We regard this as an $S_n$-invariant function in the following way. We regard $x_1,...,x_N$ as $n$ $d$-dimensional vectors $v_1,...,v_n$ and then we can give the permutation action of $S_n$ on $(v_1,...,v_n)$. This induces the action of $S_n$ on $(x_1,...,x_N)$.
A function $(x_1,...,x_N) \mapsto y = \sum_{i=1}^{N}x_i$ is invariant to the permutation actions of $S_n$.

We solve the regression problem by \textit{DeepSets} \citep{zaheer2017deep}, which is an $S_n$-invariant DNN with given $n$. 
DeepSets consists of $S_n$-equivariant layers (the first three layers), an $S_n$-invariant layer, and a fully connected layer (the last layer).
A number of units of each layer is as follows: $N \rightarrow 128 \rightarrow 64 \rightarrow 32\rightarrow 32 \rightarrow 1$.

We vary $n=2,4,6,8$ and set $N=48$, then we consider configurations $(n,d) \in \{(2,24), (4,12), (6,8), (8,6)\}$ such as satisfying $N=nd$.
We generate $m=60$ samples for training and $10000$ samples for testing.  
We train \textit{DeepSets} with $500$ epochs, batch size $4$, learning rate $0.001$, and the Adam optimizer.

Figure~\ref{fig:result} illustrates the result, which shows the mean over five trials with different random seeds.\footnote{A reviewer suggested that the experimental result should be expressed in a table. 
However, since the table is not suitable to express the slope of the gaps with respect to $n$, we continue to use the figure.}  
The horizontal axis shows $n$ and the vertical
axis shows a logarithm of the generalization gaps.
From the result, we can confirm two things.
First, the theoretical bound is a certain upper bound of the experimental value by DeepSets.
Second, the slope of the
experimental value is same to the theoretical slope. 
This supports our claim that the degree of invariance $n$ reduces the generalization gap.



 \begin{figure}
     \centering 
     \includegraphics[width=0.9 \hsize]{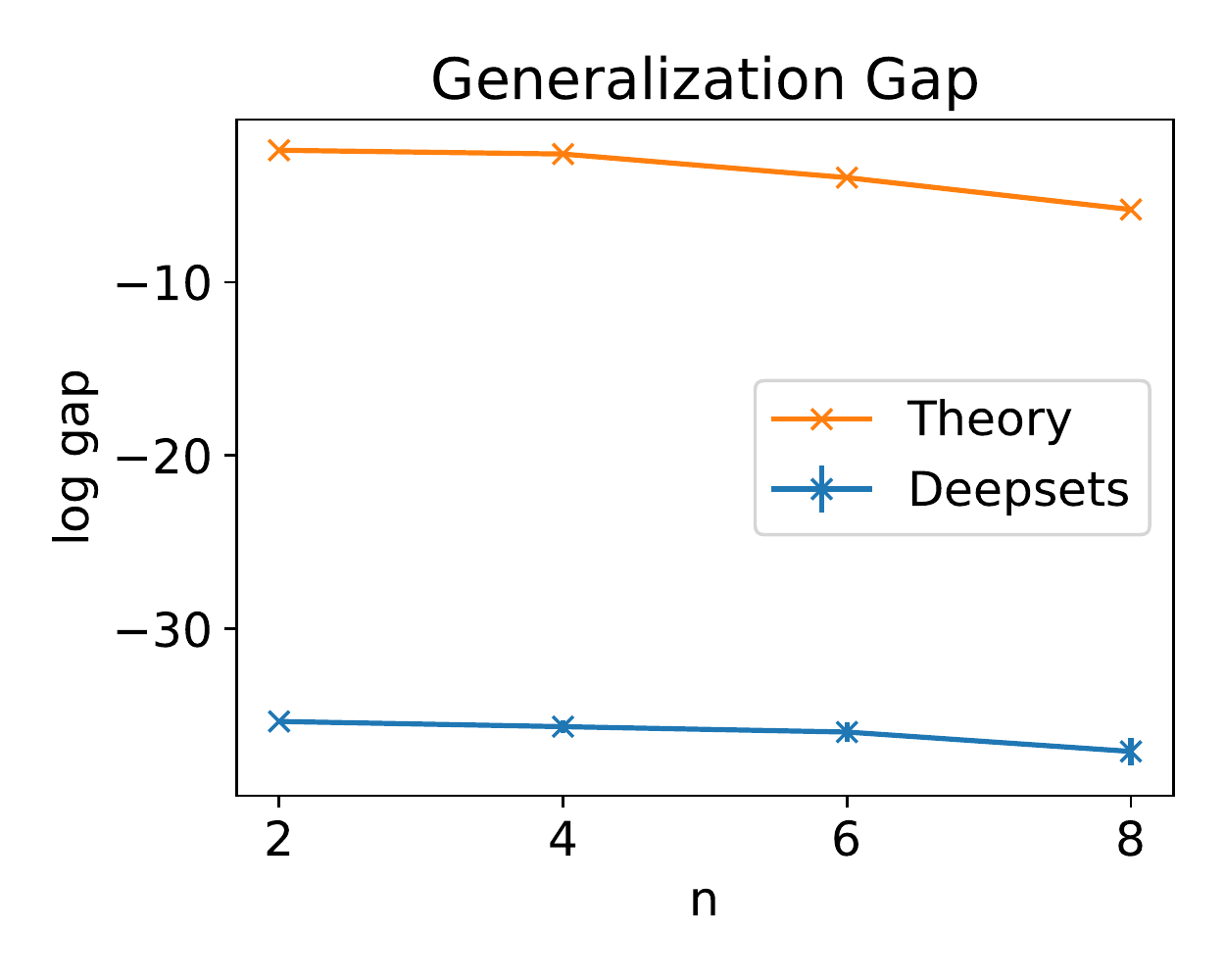} 
     \caption{Generalization Gap. Theory bound v.s. experimental result.  \textit{Theory} (orange line) denotes $( 1/(\sqrt{n! m^{2/n}}))$ , and \textit{Deepsets} (blue line) denote the generalization gaps in the experiments with $n \in \{2,  4, 6, 8\}$.}
     \label{fig:result}
 \end{figure}

\section{Discussion and Comparison} \label{sec:comparison}

\subsection{Effect of Invariance / Equivariance}

We identify the effect of invariance / equivariance on the generalization bounds of DNNs.
With invariant / equivariant DNNs, the volume of the corresponding QFS decreases, hence the generalization bounds also decrease.
In order for the bounds to reflect the volume of the QFS, its convergence rate in $m$ must be slow.
However, because the volume of the QFS decays sufficiently fast in $n$, the generalization bound decays rapidly as $n$ increases.
Figure \ref{fig:curve_log} shows a logarithmic version of the bound against $m$ in Figure \ref{fig:bound_oridinal}.


\begin{figure}[htbp]
    \centering
    \includegraphics[width=0.8\hsize]{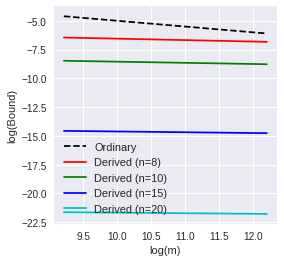}
    \caption{Logarithmic order of the bound for the generalization gap against $\log(m)$. \textit{Ordinary} (dashed line) denotes $\log(1/\sqrt{m})$ without invariance, and \textit{Derived} (colored lines) denotes the bound $\log( 1/\sqrt{n! m^{2/n}})$ with $n\in \{8,10,15,20\}$. The bound decreases sharply by $n$ with $m$ increasing to $200,000 \approx \log (12)$.\label{fig:curve_log}}
\end{figure}


\subsection{Relation with Sokolic et al. [2017]} \label{sec:relation}

The closest study to our work is the generalization error analysis of invariant classifiers by  \cite{sokolic2017generalization}.
The study shows that a number of \textit{transformations} by invariance describes their generalization bound.
While the result is similar to our study, our result improves their analysis in the following two ways.

(I). Our results are valid without the division assumptions in \cite{sokolic2017generalization}. 
The study assumes that the input space $\mathcal{X}$ can be written as $\mathcal{X}_0 \times T$ using a set of transformations $T$ and a base space $\mathcal{X}_0$. 
This assumption is hard to confirm for two reasons. 
First, the set of transformations is not the group itself in general. 
For example, consider the trivial action of a symmetric group, then despite the fact that the group is a symmetric group, the set of transformations is a single point set consisting of identities. 
Thus, we need to calculate the set of transformations on a case-by-case analysis. 
Second, it is difficult to find the base space. 
In the graph neural network case \citep{maron2018invariant}, the action is the permutation of nodes on adjacency matrices. 
In this case, it is hard to find the base space of this action. 
In contrast, our result is valid in this case.

(II).
Our results relax the stability assumption and achieve an accurate result.
\cite{sokolic2017generalization} places an algorithmic stability assumption that allows them to ignore the complexity of hypothesis spaces and build a theory solely on the complexity of the input space.
However, this stability assumption is not satisfied by deep learning in particular \citep{hoffer2017train,nagarajan2019uniform}.
In this study, we construct a theory that is independent of stability assumptions by connecting the functional hypothesis space and the theory of QFSs.

\subsection{ Analysis for Expressive Power of Invariant DNNs}

We discuss an expressive power of invariant neural networks, which determines how small the empirical loss $R_m(f^G)$ would be.
A volume of $R_m(f^G)$ is not the main concern of this study, but it is an important factor for the actual performance.


We investigate the expressive power of invariant deep neural networks.
To the aim, we define the \textit{H\"older space}, which is a class of smooth functions. 

\begin{definition}[H\"older space]
Let $\alpha > 0$ be a degree of smoothness.
For $f:I \to \R$, the \textit{H\"older norm} is 
\begin{align*}
   \|f\|_{\gH^\alpha} :=
   &\max_{\beta: |\beta| < \floor{\alpha}} \|\partial^\beta f(x)\|_{L^\infty(I)} \\
   &+  \max_{\beta = \floor{\alpha}} \sup_{x, x' \in I, x \neq x'} \frac{ | \partial^\beta f(x) - \partial^{\beta} f(x')| }{\|x - x'\|_{\infty}^{ \alpha - \floor{\alpha} } } ,
\end{align*}
A $B$-radius closed ball in the \textit{H\"older space} on $I$ is defined as
    $\gH^\alpha_{B} = \{f \in \gH^\alpha \mid \|f\|_{\gH^\alpha} \leq B\}$.

\end{definition}
$\gH^\alpha$ is a set of bounded functions that are $\alpha$-times differentiable.
The notion is often utilized in characterizing the expressive power of DNNs (e.g., refer to \cite{schmidt2017nonparametric}).
Then, we achieve the following result of an expressive power of invariant DNNs:
\begin{theorem}[Approximation rate of invariant DNNs] \label{thm:main2}
    For any $\varepsilon>0$, suppose $\gF^{S_n}$ has at most  $\mathcal{O}(\log(1/\varepsilon))$ layers and $\mathcal{O}(\varepsilon^{-D/\alpha}\log(1/\varepsilon))$ non-zero parameters.
    Then, for any $S_n$-invariant $f^*\in \gH^\alpha_B$, there exists $f^{S_n}\in \gF^{S_n}$ such that 
    \begin{align*}
        \|f^{S_n}-f^*\|_{L^{\infty}(I)} \leq \varepsilon.
    \end{align*}
\end{theorem}
Theorem \ref{thm:main2} clarifies the expressive power of DNNs by showing the sufficient number of parameters to make the error arbitrarily small.
This result shows that the error decreases as the number of parameters increase with the rate $-D/\alpha$ up to logarithm factors.
Importantly, this rate is the optimal rate \textit{without} invariance \citep{yarotsky2017error}.
Hence, we prove that invariant DNNs can achieve the optimal approximation rate even with invariance.

\section{Conclusion} \label{sec:conclusion}

We proposed the generalization theory, which describes the errors and the effect of invariance / equivariance in a quantitative way.
We proved that the order of invariance improves the generalization bound.
Moreover, we prove $S_n$-invariant DNNs maintain the high expressive power regardless of the invariance property.



\begin{contributions} 
 AS conceived the idea of quotient feature spaces and wrote the paper.
    MI computed the covering of the function space and wrote the paper.
    KM created the code.
\end{contributions}

\begin{acknowledgements} 
The authors would like to thank the anonymous reviewers for their suggestions and helpful comments. 
The first author would like to thank Professor Kazuaki Miyatani for sharing his knowledge of pseudo distance induced on quotient spaces.
AS was supported t by the JSPS Grant-in-Aid for Scientific Research C (20K03743). 
MI was supported by JSPS KAKENHI (18K18114) and JST Presto (JPMJPR1852).
KM was supported by the JSPS Grant-in-Aid for Early-Career Scientists (20K19930).
\end{acknowledgements}

\bibliography{sannai_229}

\newpage
\onecolumn
\appendix

\section{Supplement} \label{sec:proofs}

We provide the deferred proofs of each section.

\subsection{Proof for Section 3}

\begin{proof}[Proof of Lemma \ref{lem:delta}]


For any $\epsilon>0$, there is a sequence of rational numbers $\{p_i/q_i\}$ such that $p_i/q_i< \epsilon$ and converges to $\epsilon$. Assume that Lemma \ref{lem:delta} holds for rational numbers, then we have$\gN_{\varepsilon,\infty}( \Delta_{S_n}) \leq \gN_{p_i/q_i,\infty}( \Delta_{S_n}) \leq C /  (n!~{(p_i}/q_i)^{n} ).$ 
Since $1/x^n$ is a continuous function and $\{p_i/q_i\}$ converges to $\epsilon$, we obtain $\gN_{\varepsilon,\infty}( \Delta_{S_n}) \leq C /  (n!~{(\varepsilon)^{n}} ).$ Hence it is enough to show the case of rational numbers.

We assume $\varepsilon = p/q$ for some integers $p,q>0$.  Let $\gC(I)$ be the covering of $I$, which is a set of  $\varepsilon$-cubes 
\begin{align*}
    c_{j_1,..,j_n}= \{x=(x_i)\in I \mid \varepsilon j_i \leq x_i \leq \varepsilon (j_i +1)\},
\end{align*}
for $j_i = 1,.., [q/p]+1$.
We can easily see that  $\gC(I)$  attains the minimum number of  $\varepsilon$-cubes covering $I$ and the number is $(\varepsilon^{-1}+1)^{n}=\frac{\varepsilon^{-n}}{n!} + \mathcal{O}(\varepsilon^{-(n-1)})$.  We show that we can find a subset of  $\gC(I)$ which cover $\Delta_{S_n}$  and whose cardinality is  $\frac{\varepsilon^{-n}}{n!} + \mathcal{O}(\varepsilon^{-(n-1)})$. The proof is as follows. At first, we calculate the number $A$ of cubes in $\gC(I)$ which intersect with the boundary of  $\sigma \cdot \Delta$. 
Then since the number of the orbit of the cubes which do not intersect with the boundary of  $\sigma \cdot \Delta$ is $n!$, if $A$ is $\mathcal{O}(\varepsilon^{-(n-1)})$, we can find the covering whose cardinality is $\frac{\varepsilon^{-n}}{n!} + \mathcal{O}(\varepsilon^{-(n-1)})$.
Since $\sigma \cdot \Delta$  is $\left\{ x \in I \mid x_{\sigma^{-1}(1)} \geq x_{\sigma^{-1}(2)} \geq \cdots \geq x_{\sigma^{-1}(n)}  \right\}$,
any boundary of $\sigma \cdot \Delta$ is of the form
$\left\{ x \in I  \mid x_{\sigma^{-1}(1)} \geq \cdots  x_{\sigma^{-1}(i)} = x_{ \sigma^{-1}(i+1)}\geq \cdots \geq x_{\sigma^{-1}(n)}  \right\}.$

From here, we fix $\sigma$ and $i$. 
Consider the canonical projection $\pi : \R^{ n}\to \R^{ n -1}$ which sends $x_{ \sigma^{-1}(i)}$-axis to zero. $\pi$ induces the map $\tilde{\pi} : \gC(I) \to \gC(\pi(I))$, where $\gC(\pi(I))$ is the covering of  $\pi(I)$ .  
Let $\gC(I)_B$ denote the subset of cubes in $\gC(I)$ which intersect with the set $B=\left\{ x \in I \mid  x_{\sigma^{-1}(i)} = x_{\sigma^{-1}(i+1)} \right\}.$ Then we can see that $\tilde{\pi}$ is injective on $\gC(I)_B$ as follows.
Assume that there are two cubes in $\gC(I)_B$ whose images by $\tilde{\pi}$ are equal.  Let us denote the centers of two cubes by  $c_{j_1,..,j_n}$ and $c_{k_1,..,k_n}$ .Then, since $\pi$ only kills $x_{\sigma^{-1} (i)}$, $j_p=k_p$ holds for $p \neq \sigma^{-1}(i) $.  But since $c_{j_1,..,j_n}$ and $c_{k_1,..,k_n}$ are in  $\gC(I)_B$, we have $j_{\sigma^{-1}(i)}=j_{\sigma^{-1}(i+1)}$  and $k_{\sigma^{-1}(i)}=k_{\sigma^{-1}(i+1)}$.  Hence $j_p=k_p$ for any $p$ and  $\tilde{\pi}$ is injective on $\gC(I)_B$.

Next, let $\gC(I)_{\tilde{B}}$ be the subset of $\varepsilon$-cubes in $\gC(I)$ which intersect a boundary of $\sigma \cdot \Delta$ for some $\sigma$.  We see that the cardinality of $\gC(I)_{\tilde{B}}$ is bounded by $e\varepsilon^{-(n-1)}$ for some constant $e>0$.
Since the number of components of the boundaries is finite, we prove the claim for a component $B$.  As we see before, $\tilde{\pi}_{\restriction \gC(I)_
B}$ is injective. This result implies that the number of cubes that intersect $B$ is bounded by a number of $\varepsilon$-cubes in $\gC(p(I))=\varepsilon^{-(n-1)}$.
Put $\gC(I)_{F}= \gC(I)-  \gC(I)_{\tilde{B}}$. Then we note that the action of $S_n$ on$\gC (I)_F$ is free, namely the number of the orbit of any cube in $\gC (I)_F$ is $|S_n|$. Hence,$$| \gC(I)_F \cap \Delta |= | \{c \in \gC(i) \mid c \subset \Delta \}| =  1/|S_n|| \gC (I)_F |\leq \varepsilon^{-n}/|S_n|.$$ 
Here, $ (\gC(I)_F \cap \Delta)\cup \gC(I)_{\tilde{B}}$ gives the covering of $\Delta$. This covering gives
\begin{align*}
        \gN_{\varepsilon,\infty}( \Delta) \leq \frac{\varepsilon^{-n}}{n!} + e\varepsilon^{-(n-1)}.
    \end{align*}
\end{proof}

\begin{proof}[Proof of Lemma \ref{lem:covering_g}]
By Proposition \ref{prop:surj}, we have $\tilde{\Delta}_G$ satisfying two conditions above.
Since the covering of $\tilde{\Delta}_G$ induces the covering of $\Delta_G$ by the condition 2, $ \gN_{\varepsilon,\infty}( \Delta_G) \leq \gN_{\varepsilon,\infty}( \tilde{\Delta}_G). $ 
On the other hand, by the condition 1, we have $ \gN_{\varepsilon,\infty}( \tilde{\Delta}_G) \leq |S_n|/|G|\cdot \gN_{\varepsilon,\infty}( \Delta_{S_n}). $
Combining with Lemma \ref{lem:delta}, we have the desired result.
\end{proof}

\begin{proof}[Proof of Proposition \ref{prop:dist}]
For the claim, assume that an action of $G$ preserves distance, namely,  $\|x-x'\|_2 = \|g(x)-g(x')\|_2$ holds.
We show that $d_G(y, y') =\inf_{x,x' \in \R^n} \{ \|x- x'\|_2 | \phi_G(x)=y, \phi_G(x')=y' \}. $
Consider the sum $\|x - b_1\|_2 + \| a_2 -  b_2\|_2 + ... + \| a_n - x'\|_2 $  and take an element $g\in G$ such that $a_2 = g\cdot b_1$.
Then, $\|x - b_1\|_2 + \|a_2- b_2\|_2 = \|x- b_1\|_2 + \| g\cdot b_1 - b_2\|_2 = \| x -  b_1\|_2 + \| b_1 - g^{-1} \cdot b_2\|_2 \geq \| x - g^{-1}\cdot b_2\|_2$.
By repeating this process, we have  $\| x - b_1\|_2 + \|a_2 - b_2\|_2 + ... + \| a_n - x'\|_2  \geq \| x - g\cdot x'\|_2$ for some $g \in G$. Hence, $d_G(y, y') =\inf_{x,x' \in I} \{ \| x - x'\|_2 | \phi_G(x)=y, \phi_G(x')=y' \}$. This implies $d_G(y, y')=0 \Rightarrow y=y'$.
\end{proof}

\begin{proof}[Proof of Proposition \ref{prop:surj}]
We confirm that $\tilde{\Delta}_G$ satisfies both the conditions 1 and 2. 
As the action of $G$ preserves the distance, $g_k: \Delta_{S_n}\to g_k \cdot \Delta_{S_n}$  is an isomorphism on metric spaces. 
Hence, condition 1 is satisfied. 

For condition 2, we consider $y \in \Delta_G = \phi_G(I)$. 
Then, there is an element $x \in I$ such that $y=\phi_G(x)$. 
As $I= \cup_{\sigma\in S_n} \sigma \cdot \Delta_{S_n}$ , there exist $\sigma \in S_n$ and $z \in \Delta_{S_n}$ such that $x=\sigma \cdot z$.  

In contrast, as $\{g_1,..,g_K | g_k \in G\}$ is a complete system of representatives of  $G\backslash S_n$, there exist $\tau \in G$ and $g_k$ such as $\sigma = \tau \cdot g_i$. Then $\phi_G(g_k z)= \phi_G(\tau \cdot g_k z)=\phi_G(x)=y$ and $g_k \cdot z \in \tilde{\Delta}_G$. 
Hence $\phi_G(\tilde{\Delta}_G)= \Delta_G$.
\end{proof}

\subsection{Proof for Section 3.2}

\begin{proof}[Proof of Proposition \ref{prop:correspondence}]
We prove $\hat{\phi}_G$ is injective and surjective. Assume $f \in C(\Delta_G)$ and put $\hat{\phi}_G (f)=f\circ \phi_G$. 
Then since $\phi_G$ is $G$-invariant, so is $\hat{\phi}_G(f).$ Also, since $\phi_G$ is surjective, $\hat{\phi}_G$ is injective.  Take $g \in C^G(I)$, then we define $f \in C(\Delta_G)$ as follows;  for any $y \in \Delta_G$, take $x\in I$ such that $\phi_G(x)=y$  and define $f(y)=g(x)$. This map is well defined because $g$ is $G$-invariant. $\hat{\phi}_G(f)(x) = f \circ \phi_G (x)= f(y)=g(x).$ Hence, we obtain the desired result.\\
Next, we prove the Lipschitz properties. Take $f \in C(\Delta_G)$ and assume $f$ is $K$-Lipschitz. Then 
for any $x, x' \in I$,
$$
d_G(\phi_G(x), \phi_G(x')) \geq K d(f(\phi_G(x)), f(\phi_G(x'))),
$$ by $K$-Lipschitz property of $f$.  By the definition of $d_G$, we have $d_G(\phi_G(x), \phi_G(x')) \leq d(x, x')$. Hence, $\hat{\phi}_G(f)$ is $K$-Lipschitz continuous. 
Conversely, assume $\hat{\phi}_G (f)$ is $K$-Lipschitz. Take any $y, y' \in I$, then for any $x, x'\in I$ satisfying $\phi_G(x)=y, \phi_G(x)=y',$ 
$$
d(x, x')\geq K d(f(\phi_G(x)), f(\phi_G(x')))=d(f(y), f(y')),
$$
by $K$-Lipschitz property of $\hat{\phi}_G (f)$. Hence by taking infimum of the left hand side, we have
$$
d_G(y, y')=\mbox{inf}\ d_G(\phi_G(x), \phi_G(x')) \geq K d(f(y), f(y')).
$$
Hence, $f$ is $K$-Lipschitz.
\end{proof}

\begin{proof}[Proof of Proposition \ref{prop:delta_sn}]
We first note that $\phi$ is the identity map on $\Delta$, because elements in $\Delta$ are sorted.  
This implies $\Delta \cong \phi(\Delta)$. 
Therefore, it is sufficient to show $\Delta \cong \Delta_{S_n}$.
As $\Delta $ is a subset of $I$, we have the distance preserving map ${\phi_{S_n}}_{\restriction_{\Delta}}: \Delta \to \Delta_{S_n}$. 
 
Then, we show that ${\phi_{S_n}}_{\restriction_{\Delta}}$ is a bijection. \textit{Injectivity}: Ley us take any $x,y \in \Delta$ such that ${\phi_{S_n}}_{\restriction_{\Delta}}(x)={\phi_{S_n}}_{\restriction_{\Delta}}(y)$. Then $x=g\cdot y$ for some $g \in S_n$.  
However, as $y$ is in $\Delta$, $\{g\cdot y|g \in S_n \} \cap \Delta = \{y\}$. Hence, $x=y$.
\textit{Surjectivity}: Take any $z\in \Delta_{S_n}$, then there is $x\in I$ such that $z=\phi_{S_n}(x)$. 
By the construction of $\Delta$, there is $g\in S_n$ and $y\in\Delta$ that satisfies $x=g\cdot y$.  
Hence, $z=\phi_{S_n}(x)=\phi_{S_n}(g\cdot y)=\phi_{S_n}(y)$.
\end{proof}

\begin{proof}[Proof of Proposition \ref{prop:covering_bound_g}]
Firstly, we show $\hat{\phi}_G^{-1}(f) \in \gF(I)$ with any $f \in \gF^{G}(I)$.
For $f \in \gF^{G}(I)$, we consider $\hat{\phi}_G^{-1}(f) \in C(\Delta_G)$ as Proposition \ref{prop:correspondence}.
Suppose $f$ and $f'$ are $K$-Lipschitz continuous, then $\hat{\phi}_G^{-1}(f)$ is also $K$-Lipschitz continuous by Proposition \ref{prop:correspondence}.
Since \cite{zhang2018tropical} states that Lipschitz continuous functions are represented by DNNs, we have $\hat{\phi}_G^{-1}(f) \in \gF(\Delta_G)$.

Fix $f_1,f_2 \in \gF^{G}(I)$.
Then, there exist $f_1',f_2' \in \gF(\Delta_G)$ such as $f_1 = \hat{\phi}_G(f_1')$ and $f_2 = \hat{\phi}_G(f_2')$.
Then, we have  \begin{align*}
    \|f_1-f_2\|_{L^\infty(I)}&= \|\hat{\phi}_G(f_1')-\hat{\phi}_G(f_1')\|_{L^\infty(I)} =\|f_1' \circ \phi_G-f_2' \circ \phi_G\|_{L^\infty(I)} \leq \|f_1'-f_2'\|_{L^\infty(\Delta_G)}.
\end{align*}
Based on the result, we can bound $ \gN_{\varepsilon,\infty}({\gF}^{G}(I))$ by $ \gN_{\varepsilon,\infty}({\gF}(\Delta_G))$. 
Let us define $ N := \gN_{\varepsilon,\infty}({\gF}(\Delta_G))$.
Then, there exist $f'_1,...,f'_N$ such that for any $f'\in {\gF}(\Delta_G)$, there exists $j \in \{1,...,N\}$ such as $\|f'_j - f'\|_{L^\infty(\Delta_G)} \leq \varepsilon$.
Here, for any $f \in \gF^{G}(I)$, there exists $f_j := \hat{\phi}_G^{-1}(f'_j)\in \gF^{G}(I)$ and it satisfies $\|f - f_j\|_{L^\infty(I)} \leq \| \hat{\phi}_G(f)- \hat{\phi}_G(f_j)\|_{L^\infty(\Delta_G)} \leq \varepsilon$.
Then, we obtain the statement.
\end{proof}

\begin{proof}[Proof of Theorem \ref{thm:main1}]
Combining Proposition \ref{prop:covering_bound_g} and \ref{prop:covering_bound_delta}, we obtain a bound for $\log \gN_{2 C_\Delta \delta,\infty}( \gF^{G}(I))$.
Then, we substitute it into \eqref{ineq:basic_bound} and obtain the statement of Theorem \ref{thm:main1}.
\end{proof}

\subsection{Proof for Section 4}

\begin{proof}[Proof of Proposition \ref{prop:covering_bound_delta}]
We bound a covering number of a set of $C_\Delta$-Lipschitz continuous functions on $\Delta$.
Let $\{x_1,...,x_K\} \subset \Delta$ by a set of centers of $\delta$-covering set for $\Delta$.
By Lemma \ref{lem:delta}, we set $K = C / (|G|~ \delta^{n})$ with $\delta$ with a parameter $\delta > 0$, where $C>0$ is a constant.

We will define a set of vectors to bound the covering number.
We define a discretization operator $A:\gF(\Delta_G) \to \R^K$ as
\begin{align*}
    Af = (f(x_1)/\delta,...,f(x_K)/\delta)^\top.
\end{align*}
Let $\gB_\delta(x)$ be a ball with radius $\delta$ in terms of the $\|\cdot\|_\infty$-norm.
For two functions $f,f' \in \gF(\Delta_G)$ such as $Af = Af'$, we obtain
\begin{align*}
    \|f-f'\|_{L^\infty(I)} &= \max_{k =1,...,K} \sup_{x \in \gB_{\delta}(x_k)}|f(x)-f'(x)|\\
    &\leq \max_{k =1,...,K} \sup_{x \in \gB_{\delta}(x_k)}|f(x)-f(x_k)| + |f'(x_k)-f(x_k)| \\
    &\leq 2 C_\Delta \delta,
\end{align*}
where the second inequality follows $f(x_k) = f'(x_k)$ for all $k = 1,...,K$ and the last inequality follows the $C_\Delta$-Lipschitz continuity of $f$ and $f'$.
By the relation, we can claim that $\gF(\Delta_G)$ is covered by $2C_\Delta \delta$ balls whose center is characterized by a vector $b \in \R^K$ such as $b=Af$ for $f \in \gF(\Delta_G)$.
Namely, $\gN_{2 C_\Delta \delta,\infty}( \gF(\Delta_G))$ is bounded by a number of possible $b$.

Then, we construct a specific set of $b$ to cover $\gF(\Delta_G)$.
Without loss of generality, assume that $x_1,...,x_K$ are ordered satisfies such as $\|x_k - x_{k+1}\|_\infty \leq 2\delta$ for $k=1,...,K-1$.
By the definition, $f \in \gF(\Delta_G)$ satisfies  $\|f\|_{L^\infty(\Delta)} \leq B$.
$b_1 = f(x_1)$ can take values in $[-B/\delta,B/\delta]$. 
For $b_2 = f(x_2)$, since $\|x_1-x_2\|_\infty \leq 2\delta$ and hence $|f(x_1)-f(x_2)| \leq 2 C_\Delta \delta$, a possible value for $b_2$ is included in $[(b_1-2\delta)/\delta,(b_1+2\delta)/\delta]$.
Hence, $b_2$ can take a value from an interval with length $4$ given $b_1$.
Recursively, given $b_k$ for $k=1,...,K-1$, $b_{k+1}$ can take a value in an interval with length $4$.

Then, we consider a combination of the possible $b$.
Simply, we obtain the number of vectors is $(2cB/\delta) \cdot (4c)^{K-1} \leq (8c^2B/\delta)^{K-1}$ with a universal constant $c \geq 1$.
Then, we obtain that
\begin{align*}
    &\log \gN_{2 C_\Delta \delta,\infty}( \gF(\Delta_G)) \leq (K-1)\log (8c^2B/\delta).
\end{align*}
Then, we specify $K$ which describe a size of $\Delta$ through the set of covering centers.
\end{proof}

\subsection{Proof for Section 5}
\begin{proposition}\label{prop:covering_equiv_transitive}
Suppose $G$ is transitive.
Then, for any $\varepsilon >0 $, we have
\begin{align*}
    &\gN_{\varepsilon,\infty}( \tilde{\gF}^{G}(I)) \leq   \gN_{\varepsilon,\infty}( \gF^{\mathrm{St}(G)}(I)).
\end{align*}
\end{proposition}

\begin{proof}[Proof of Proposition\ref{prop:covering_equiv_transitive}]
The first statement simply follows Proposition \ref{prop:equiv_general} with setting $J=1$, since $g \in G$ is transitive.
In the case of $S_n$, we have $J=1$ and $\mbox{Stab}(1)\cong S_{n-1}$. This gives the second statement.
\end{proof}

\begin{proof}[Proof of Theorem \ref{thm:main1_equiv} and Corollary \ref{cor:sn_equiv}]
For Theorem \ref{thm:main1_equiv}, we combine the bound \eqref{ineq:basic_bound}, Lemma \ref{lem:covering_g} and
Proposition \ref{prop:covering_bound_g}.
Thus, we obtain the statement.

For Corollary \ref{cor:sn_equiv}, since $S_n$ is transitive, the statement obviously holds with $|\mathrm{St}(G)| = |S_{n-1}| = (n-1)!$.
\end{proof}

\subsection{Proof for Section 6}

To prove Theorem \ref{thm:main2}, we consider a $\mbox{Sort}$ map and show that DNNs can represent the map.
Let $\max^{(k)}(x_1,...,x_n)$ be a map which returns the $k$-th largest value of inputted elements $x_1,...,x_n$ for $k=1,..,n$.
Then, we provide a form of $\mbox{Sort}$ as
\begin{align*}
    \mbox{Sort}(x_{1},\ldots , x_{n}) = (\mbox{max}^{(1)}(x_{1},\ldots , x_{n}), \ldots , \mbox{max}^{(n)}(x_{1},\ldots , x_{n}) ).
\end{align*}
To represent it, we provide the following propositions.

\begin{proposition}\label{max} $\max^{(j)}(z_1,\ldots , z_N)$ and $\min^{(j)}(z_1,\ldots , z_N)$ are represented by an existing deep neural networks with an ReLU activation for any $j = 1,...,N$.
\end{proposition}
\begin{proof}[Proof of Proposition \ref{max}]
Firstly, since
\begin{align*}
\max(z_1, z_2) = \max(z_1-z_2, 0) + z_2, 
\end{align*}
and
\begin{align*}
\min(z_1, z_2) = -\max(z_1-z_2, 0) + z_1
\end{align*}
hold, we see the case of $j=1, N=2$.
By repeating $\max(z_1, z_2)$, we construct $\max^{(1)}(z_1,\ldots , z_N) $ and $\min^{(1)}(z_1,\ldots , z_N)$.
Namely, we prove the claim in the case of $j=1$ and arbitrary $N$. 
At first, we assume $N$ is even without loss of generality, then we divide the set $\{z_1,...z_N\}$ into sets of pairs $\{(z_1, z_2), ... (z_{N-1}, z_N)\}$. 
Then, by taking a max operation for each of the pairs, we have $\{ y_1=\max(z_1, z_2), ..., y_{N/2} = \max(z_{N-1}, z_N)\}$ .
We repeat this process to terminate. 
Then we have $\max^{(1)}(z_1, \ldots, z_N)$, which is represented by an existing deep neural network. 
Similarly, we have $\min^{(1)}(z_1,\ldots , z_N)$. 
Finally, we prove the claim on $j = 2,..., N$ by induction.
Assume that for any $N$ and $\ell<j$,  $\max^{(\ell)} (z_1,\ldots , z_N)$ is represented by a deep neural network. 
We construct  $\max^{(j)}(z_1,\ldots , z_N)$ as follows:
since  
\begin{align*}
    {\max}^{(j-1)}({z}_{-\ell}) &= \begin{cases}
    \max^{(j-1)}(z_1,\ldots , z_N) & (\mbox{if~}\  z_\ell \leq \max^{(j)}(z_1,\ldots , z_N) ) \\
    \max^{(j)}(z_1,\ldots , z_N)  & (\mbox{otherwise})
  \end{cases}
\end{align*}
holds, we have $\max^{(j)}(z_1,\ldots , z_N) = \min (\{\max^{(j-1)}({Z}_{\ell})\mid \ell = 1,...,N\})$. By inductive hypothesis, the right hand side is represented by a deep neural network.
\end{proof}

Further, we provide the following result for a technical reason.

\begin{proposition}\label{prop:bijDNN}
The restriction map
$$
\Lambda : \mathcal{F}^{S_n}(I) \to \mathcal{F}(\Delta_{S_n})
$$
is bijective, where $\Lambda(f)=  f_{\restriction_{\Delta_{S_n}}}$.
\end{proposition}

\begin{proof}[Proof of Proposition \ref{prop:correspondence}]
To show the Proposition, we firstly define sorting layers which is an $S_n$-invariant network map from $I$ to $\Delta$.
Then by Proposition \ref{max}, $\mbox{Sort}(x_{1},\ldots , x_{n}) $ is also a function by an $S_n$-invariant deep neural network and $\mbox{Sort}(x_{1},\ldots , x_{n})$ is the function from $I$ to $\Delta$. 
By using this function, we define the inverse of $\Lambda$. 
For any function $f$ by a deep neural network on $\Delta$,  we define $\Phi(f) = f \circ \mbox{Sort}$.
We confirm $\Lambda\circ \Phi = \mbox{id}_{\gF_{\Delta}}$ and $\Phi \circ \Lambda =\mbox{id}_{\gF^{S_n}}$. Since we have
\begin{align*}
\Lambda\circ \Phi (f) = \Lambda\circ f\circ \mbox{Sort} = (f\circ \mbox{Sort})_{\restriction_{\Delta}} = f,
\end{align*}
$\Lambda\circ \Phi$ is equal to  $\mbox{id}_{\gF_{\Delta}}$.
Similarly,
\begin{align*}
\Phi \circ \Lambda(f) = \Phi \circ f_{\restriction_{\Delta}} = f_{\restriction_{\Delta}} \circ \mbox{Sort} = f,
\end{align*}
where the last equality follows from the $S_n$-invariance of $f$.
Hence, we have the desired result.
\end{proof}

Now, we are ready to prove Theorem \ref{thm:main2}.
\begin{proof}[Proof of Theorem \ref{thm:main2}]
Let $f^*$ be an ${S_n}$-invariant function on $I$. Then by Proposition \ref{prop:bijDNN}, we have a function $f$ on $\Delta_{S_n}$ such that 
$f^*= f \circ \mbox{Sort} $ holds. 
By Theorem 5 in \cite{schmidt2017nonparametric}, for enough big $N$, there exists a constant $c>0$ and a neural network $f'$ with at most $\mathcal{O}(\log(N))$ layers and at most $\mathcal{O}(N\log(N))$ nonzero weights such that $\|f-f'\|_{L^{\infty}(I)} \leq cN^{-\alpha/p}$. 
Then, we have
\begin{align*}
&\| f^*-f'\circ \mbox{Sort} \|_{L^{\infty}(I)} = \|f\circ \mbox{Sort}-f'\circ \mbox{Sort} \|_{L^{\infty}(I)} \leq \|f-f' \|_{L^{\infty}(\Delta)} \leq  \|f-f' \|_{L^{\infty}(I)} \leq cN^{-\alpha/p},
\end{align*}
where $f\circ \mbox{Sort}$ is a neural network with at most $\mathcal{O}(\log(N))+K_1$ layers and at most $\mathcal{O}(N\log(N)) + K_2$ nonzero weights, where $K_1$ and $K_2$ are the number of layers and the number of  nonzero weights of the neural network expressing $\mbox{Sort}$ respectively. By replacing $N^{-1}$ with $\varepsilon$, we have the desired inequality.
\end{proof}

\section{Generalization Bound for Equivalent DNN without Transitive Assumption } \label{sec:equiv_without_transitive}

In this section, we provide a general version of the result in Section \ref{sec:gen_equiv}.
Namely, we relax the transitive assumption in the section.
To the goal, we newly define a general version of a stabilizer subgroup.

Let $[n] = \{1,2,\dots, n \}$ be an index set and $G$ be a finite group action on $[n]$.
For $i \in [n]$, we define the stabilizer subgroup $\mbox{Stab}_G(i)$ associated with $G$ as
\begin{align*}
    \mbox{Stab}_G(i) = \left\{\sigma\in G \mid \sigma\cdot i = i \right\}.
\end{align*}
We also consider the following decomposition of $[n]$ as 
$$
[n]= \bigsqcup_{j \in \mathcal{J}} \mathcal{O}_{j},
$$ 
where $\mathcal{J} \subset I$ and $\mathcal{O}_{j}$ is a $G$-orbit of $j$, namely the set of the form $G\cdot j$. 
Any $G$-orbit $G\cdot j$ is isomorphic to the set $G/\mbox{Stab}(j)$. We denote $|\mathcal{J}|$ by $J$ and $|\mathcal{O}_j|$ by $l_j$.
For each $j \in \mathcal{J}$, let 
$
 G = \bigsqcup_{j \in \mathcal{J}} \bigsqcup_{k=1}^{l_j}\mbox{Stab}_G(j) \tau_{j,k}
$
be the coset decomposition by $\mbox{Stab}_G(j)$. Then, we may assume that 
$\tau_{j,k}\in G$ satisfies $\tau_{j,k}^{-1}(j) = j + k$. 

Then, we provide another representation for equivariant functions from the following study.:
\begin{proposition}[Representation for Equivariant Functions \cite{sannai2019universal}]\label{prop:relation-invariant-equivariant}
A map $F\colon\mathbb{R}^n \to \mathbb{R}^n$ is $G$-equivariant if and only if $F$ can be 
represented by $F = (f_1\circ \tau_{1,1}, f_1\circ \tau_{1,2}, \dots, f_1\circ \tau_{1,l_1}, f_2\circ\tau_{2,1} \dots, f_{J}\circ \tau_{J,l_{J}})^\top$ for some $\mbox{Stab}_G(j)$-invariant functions 
$f_j\colon\mathbb{R}^n \to \mathbb{R}$.
Here, $\tau_{j,k} \in G$ is regarded as a linear map $\mathbb{R}^n\to \mathbb{R}^n$. 
\end{proposition}

\begin{proposition} \label{prop:equiv_general}
For any $\varepsilon >0 $, we have
\begin{align*}
    &\tilde{\gN}_{\varepsilon,\infty}( \tilde{\gF}^{G}(I)) \leq  \prod_{j \in \mathcal{J}} \gN_{\varepsilon,\infty}( \gF^{\mbox{Stab}_G(j)}(I_{l_j})),
\end{align*}
where $I_{l_j}=[0,1]^{l_j}$.
Further, if $G=S_n$,
\[
  \tilde{\gN}_{\varepsilon,\infty}( \tilde{\gF}^{S_n}(I))  \leq \gN_{\varepsilon,\infty}( \gF^{S_{n-1}}(I)). 
\]
\end{proposition}
 \begin{proof}[Proof of Proposition \ref{prop:equiv_general}]
 We put $N_j =\gN_{\varepsilon,\infty}( \gF^{\mbox{Stab}_G(j)}(I))$.
 For each $j \in \mathcal{J}$, by the definition of covering numbers, there exist $f_j^{(1)},..,f_j^{(N_j)} \in \gF^{\mathrm{Stab}_G{(j)}}(I_{l_j})$ such that for any $f' \in \gF^{\mathrm{Stab}_G{(j)}}(I_{l_j})$, there exists $f^{(p)}_j$ satisfying $\|f' - f^{(p)}_j\|_{\infty} < \varepsilon$.
 
With a tuple $(p_1,...,p_J)$, we consider a map $F_{p_1,..,p_{J}}: I \to \R^n$ from $\tilde{\gF}^G(I)$ and claim that balls $\gB_{\varepsilon}( F_{p_1,..,p_J})$ give a covering set of $\tilde{\gF}(I)$.
Put $F_{p_1,..,p_{J}} = (f_1^{(p_1)}\circ \tau_{1,1}, f_1^{(p_1)}\circ \tau_{1,2}, \dots, f_1^{p_1}\circ \tau_{1,l_1}, f_2^{(p_2)}\circ\tau_{2,1} \dots, f_{J}^{(p_{J})}\circ \tau_{J,l_{J}})^\top.$ 
Then $F_{p_1,..,p_{J}}$ is a $G$-equivariant map.
Also, since $\tau_{j,k}$ is a linear map by Proposition \ref{prop:relation-invariant-equivariant}, we can represent $\tau_{j,k}$ by DNNs.
Hence, $F_{p_1,..,p_{J}} \in \tilde{\gF}^G(I)$ holds.

Fix $F' \in \tilde{\gF}^G(I)$ arbitrary.
We have the representation $F' = (f'_1\circ \tau_{1,1}, f'_1\circ \tau_{1,2}, \dots, f'_1\circ \tau_{1,l_1}, f'_2\circ\tau_{2,1} \dots, f'_{J}\circ \tau_{j,l_{J}})^\top$ by Proposition\ref{prop:relation-invariant-equivariant}. 
Then, we can find a corresponding $F_{p_1,..,p_{J}}$ such as
\begin{align*}
    \vertiii{F_{p_1,..,p_J}- F'}_{L^\infty(I)}&= \max\{\|f^{(p_j)}_j\circ \tau_{j,k_j}-f'_{j}\circ \tau_{j,k_j}\|_{\infty}  \mid\ 1\leq k_j \leq |G/ \mbox{Stab}_G(j)|, 1\leq p_j \leq N_j \}\\
    &=\max\{\|f^{(p_j)}_j-f'_{j}\|_{\infty} \mid\  1\leq p_j \leq N_j \}\\
    &\leq \varepsilon.
 \end{align*}
Hence, we have the first statement.

In the case of $S_n$, we have $J=1$ and $\mbox{Stab}(1)\cong S_{n-1}$. This gives the second statement.
\end{proof}

Then, we obtain the following general bound:
\begin{theorem}[Generalization of Equivariant DNN] \label{thm:main1_equiv_no_transitive}
    Suppose $\tilde{f}^G \in \tilde{\gF}^{G}(I)$ is uniformly bounded by $1$.
    Then, for any $\varepsilon > 0$, the following inequality holds with probability at least $1-2\varepsilon$:
    \begin{align*}
        &R(\tilde{f}^G) \leq R_m(\tilde{f}^G)  + {\sqrt{ \sum_{j\in \mathcal{J}}\frac{ \tilde{c} }{|\mbox{Stab}_G(j)|~m^{2/n} }}}+ {\sqrt{\frac{2\log  ( 2/\varepsilon )}{m}}}.
    \end{align*}
    where $\tilde{c} > 0$ is a constant which are independent of $n$ and $m$.
\end{theorem}
We omit rigorous proof of Theorem \eqref{thm:main1_equiv_no_transitive}, because it is almost same to that of Theorem \ref{thm:main1_equiv}.

\end{document}